\documentclass{article}
\usepackage{journal}
\usepackage{natbib}
\usepackage[mathscr]{eucal}
\usepackage[usenames,dvipsnames]{xcolor}
\usepackage[backref,pageanchor=true,plainpages=false,pdfpagelabels,bookmarks,bookmarksnumbered,breaklinks,pdfborder={0 0 0},colorlinks=true,
	citecolor=Sepia,
	linkcolor=Sepia,
	filecolor=Sepia,
	urlcolor=Sepia
]{hyperref}

\usepackage{framed}
\usepackage{amsmath,amsfonts,amsbsy,amsgen,amscd,mathrsfs,amssymb}

\usepackage{amsthm}
\usepackage{mathtools}
\usepackage[font=small,margin=0.25in,labelfont={sc},labelsep={colon}]{caption}
\usepackage[shortlabels]{enumitem}

\makeatletter
\def\set@curr@file#1{\def\@curr@file{#1}} 
\makeatother
\usepackage[load-configurations=version-1]{siunitx} 

\usepackage{tikz}
\usepackage{mathdots}
\usepackage{yhmath}
\usepackage{cancel}
\usepackage{color}
\usepackage{array}
\usepackage{multirow}
\usetikzlibrary{fadings}
\usetikzlibrary{patterns}
\usetikzlibrary{shadows.blur}
\usetikzlibrary{shapes}

\usepackage[colorinlistoftodos]{todonotes}
\usepackage[utf8]{inputenc} 
\usepackage[T1]{fontenc}    
\usepackage{hyperref}       
\usepackage{booktabs}       
\usepackage{amsfonts}       
\usepackage{nicefrac}       
\usepackage{microtype}      
\usepackage[framemethod=TikZ]{mdframed}
\mdfsetup{skipabove=\topskip,skipbelow=\topskip}

\usepackage{bm} 
\usepackage{bbm}
\usepackage{amssymb}
\usepackage{xspace}
\usepackage[algo2e,ruled,linesnumbered,vlined]{algorithm2e}

\newtheorem{thm}{Theorem}
\newtheorem{lem}[thm]{Lemma}
\newtheorem{cor}[thm]{Corollary}

\theoremstyle{definition}

\newtheorem{defn}{Definition}
\newtheorem*{defn*}{Definition}

\renewenvironment{proof}[1][]{\par\noindent{\bf Proof #1\ }}{\hfill\BlackBox\\[2mm]}

\newcommand{\inbrace}[1]{\left \{ #1 \right \}}
\newcommand{\inparen}[1]{\left ( #1 \right )}
\newcommand{\insquare}[1]{\left [ #1 \right ]}

\newcommand{\ceil}[1]{\left \lceil #1 \right \rceil}
\newcommand{\floor}[1]{\left \lfloor #1 \right \rfloor}

\newcommand{\sett}[1]{\inbrace{#1}}

\DeclareMathOperator*{\argmin}{argmin}
\DeclareMathOperator*{\argmax}{argmax}

\DeclareMathOperator*{\Ex}{\mathbb{E}}

\DeclareMathOperator*{\Prob}{Pr}

\newcommand{\eps}{\varepsilon}

\newcommand{\bbN}{{\mathbb N}}

\newcommand{\bbA}{{\mathbb A}}

\let\boldm\bm

\newcommand{\bu}{{\boldm u}}

\newcommand{\by}{{\boldm y}}

\newcommand{\calA}{\mathcal{A}}
\newcommand{\calB}{\mathcal{B}}
\newcommand{\calC}{\mathcal{C}}
\newcommand{\calD}{\mathcal{D}}

\newcommand{\calG}{\mathcal{G}}
\newcommand{\calH}{\mathcal{H}}

\newcommand{\calN}{\mathcal{N}}
\newcommand{\calO}{\mathcal{O}}

\newcommand{\calU}{\mathcal{U}}

\newcommand{\calX}{\mathcal{X}}
\newcommand{\calY}{\mathcal{Y}}

\newcommand{\sfO}{{\mathsf{O}}}

\newcommand{\RERM}{{\rm RERM}}
\newcommand{\ind}{\mathbbm{1}}
\newcommand{\Risk}{{\rm R}}
\newcommand{\err}{{\rm err}}
\newcommand{\vc}{{\rm vc}}
\newcommand{\Tdim}{{\rm Tdim}}
\newcommand{\Ldim}{{\rm lit}}
\newcommand{\MAJ}{{\rm MAJ}}
\newcommand{\co}{{\rm co}}
\newcommand{\im}{{\rm im}}
\newcommand{\alg}[1]{\textup{\textsf{#1}}\xspace}
\newcommand{\Mre}{m_0}

\newcommand{\smodel}{{Perfect Attack Oracle model}\xspace}
\newcommand{\cycalg}{{CycleRobust}\xspace}
\newcommand{\rlua}{{RLUA}\xspace}

\newcommand{\removed}[1]{}

\newcommand{\abs}[1]{\left\lvert #1 \right\rvert}

\usepackage{marginnote}

\newcommand{\todol}[2][]{{%
 \let\marginpar\marginnote
 \reversemarginpar
 \renewcommand{\baselinestretch}{0.8}%
 \todo[#1]{#2}}}

\usepackage{prettyref}
\newrefformat{fig}{Figure~\ref{#1}}
\newrefformat{alg}{Algorithm~\ref{#1}}
\newrefformat{def}{Definition~\ref{#1}}
\newrefformat{eqn}{Equation~\ref{#1}}
\newrefformat{cor}{Corollary~\ref{#1}}
\newrefformat{app}{Appendix~\ref{#1}}
\newrefformat{rem}{Remark~\ref{#1}}
\newrefformat{assu}{Assumption~\ref{#1}}

\jmlrheading{}{}{}{}{}{}

\ShortHeadings{}{}

\begin{document}

\title{\large Adversarially Robust Learning with Unknown Perturbation Sets}
\author{%
 \name{Omar Montasser} \email{omar@ttic.edu}\\
 \name{Steve Hanneke} \email{steve.hanneke@gmail.com}\\
 \name{Nathan Srebro} \email{nati@ttic.edu}\\
 \addr Toyota Technological Institute at Chicago, Chicago IL, USA
}

\maketitle

\begin{abstract}%
We study the problem of learning predictors that are robust to adversarial examples with respect to an unknown perturbation set, relying instead on interaction with an adversarial attacker or access to attack oracles, examining different models for such interactions. We obtain upper bounds on the sample complexity and upper and lower bounds on the number of required interactions, or number of successful attacks, in different interaction models, in terms of the VC and Littlestone dimensions of the hypothesis class of predictors, and without any assumptions on the perturbation set.
\end{abstract}


\section{Introduction}
We consider the problem of learning predictors that are \emph{robust} to adversarial corruptions at test time. Given an instance space $\calX$ and label space $\calY=\sett{\pm1}$, we would like to be robust against some \emph{perturbation set} $\calU:\calX \to 2^{\calX}$, where $\calU(x)\subseteq \calX$ represents the set of possible corruptions of $x$. 

Almost all prior work on adversarial robustness starts with specifying a perturbation set $\calU$ we would like to be robust against.  The type of perturbation sets we are truly interested in are often sets $\calU$ that capture ``natural'' or ``imperecptible'' perturbations.  But partially because of the need to specify $\calU$ explicitly during training, simpler sets are often used, such as $\ell_p$-norm balls \citep{DBLP:journals/corr/GoodfellowSS14}, or orbits w.r.t.~translations and rotations \citep{pmlr-v97-engstrom19a}.  Furthermore, training procedures are often specific to the perturbation set $\calU$, or have the perturbation set ``hard coded'' inside them. Some methods rely on predictor implementations that need to ``know'' the specific perturbation set $\calU$ at test-time \citep[e.g.,~randomized smoothing][]{DBLP:conf/sp/LecuyerAG0J19,DBLP:conf/icml/CohenRK19,DBLP:conf/nips/SalmanLRZZBY19}, and some methods use ``explicit'' knowledge of $\calU$ only during training-time \citep[e.g.,~][]{wong2018provable,raghunathan2018certified,raghunathan2018semidefinite,pmlr-v99-montasser19a}. 

\begin{mdframed}[frametitle={\colorbox{white}{Main Question:}},
frametitleaboveskip=-\ht\strutbox,
frametitlealignment=\center,
linecolor=Sepia!30,linewidth=1pt
outerlinewidth=1pt,
skipabove=5pt,
skipbelow=0pt]
\vskip -3mm
\centering
Can we design robust learning algorithms that do not require explicit knowledge of the adversarial perturbations $\calU$?\\ What reasonable models of access to, or interactions with, $\calU$ could we rely on instead?
\end{mdframed}

In this paper, we ask whether it is possible to develop generic learning algorithms with robustness guarantees, without knowing the perturbation set $\calU$ a-priori. That is, we want to design general robust algorithmic frameworks that work for any perturbation set $\calU$, given a \emph{reasonable} form of access to $\calU$, and avoid algorithms tailored to a specific $\calU$ such as $\ell_\infty$ or $\ell_2$ perturbations. This is important if we want to be able to easily adapt our training procedures to different perturbation sets, or would like to build ML systems that are robust to fairly abstract perturbation sets $\calU$ such as ``images that are indistinguishable to the human eye'' \citep*[see e.g.,][]{DBLP:journals/corr/abs-2006-12655}.  In our frameworks, instead of redesigning or reprogramming the training algorithm, one would only need to implement or provide specific ``attack procedures'' for $\calU$.

\newpage
In this paper, we consider robustly learning a hypothesis class $\calH\subseteq\calY^\calX$ (e.g.,~neural networks). The learning algorithm receives as input $m$ iid samples $S=\sett{(x_i,y_i)}_{i=1}^{m}$ drawn from an unknown distribution $\calD$ over $\calX\times \calY$. A predictor $h:\calX\to\calY$ is robustly correct on an example $(x,y)$ w.r.t.~$\calU$ if $\sup_{z\in\calU(x)}\ind[h(z)\neq y]=0$. The learning algorithm has no \emph{explicit} knowledge of $\calU$, but instead is allowed the following forms of access: 

\paragraph{Access to a (perfect) adversarial attack oracle}

In this model, the learning algorithm has access to a ``mistake oracle'', which we can also think of as a perfect attack oracle for $\calU$. A perfect attack oracle for $\calU$ receives as input a predictor $g:\calX\to\calY$ and a labeled example $(x,y)$, and is asked to either assert that $g$ is robustly correct on $(x,y)$, or return a perturbation $z\in\calU(x)$ that is miss-classified (see \prettyref{def:mistake-oracle}). The learning algorithm can query the perfect attack oracle for $\calU$ by calling it $T$ times with queries of the form: $(g_t, (x'_t,y'_t))$, where $g_t$ is a predictor and $(x'_t,y'_t)$ is a labeled example (not necessarily from the training set $S$). The goal of the learning algorithm is to output a predictor $\hat{h}$ with small \emph{robust} risk 
\begin{equation}
   \label{eqn:rob-risk}
    \Risk_{\calU}(\hat{h};\calD) \triangleq \Ex_{(x,y) \sim \calD}\!\left[ \sup\limits_{z\in\calU(x)} \ind[\hat{h}(z)\neq y] \right].
\end{equation}

In \prettyref{sec:mistakeoracle}, we present algorithms, guarantees on the required sample complexity and number of oracle accesses, and lower bounds on the required number of accesses, for robustly learning $\calH$ in the \smodel. These results are summarized in Table~\ref{tab:mistakeoracle-results}. 

\begin{table}
\begin{tabular}{@{}llll@{}}
\toprule
                            & Sample Complexity & Oracle Complexity &   \\ \midrule
\multirow{3}{*}{Realizable} & $\Tilde{O}(\Ldim(\calH))$               & $\Tilde{O}(\Ldim(\calH))$               & \citet{DBLP:conf/nips/MontasserHS20}. \\
                            & $\Tilde{O}(\vc(\calH){\vc^*}^2(\calH))$               & $2^{\Tilde{O}(\vc^2(\calH){\vc^*}^2(\calH))}\Ldim(\calH)$               & New result in this paper (\prettyref{thm:realizable-strongust}). \\
                            &                   & $\Omega(\Ldim(\calH))$               &   New result in this paper (\prettyref{thm:lowerbound-oraclecomplexity}).\\ \midrule
Agnostic                    & $\Tilde{O}(\Ldim(\calH))$               & $\Tilde{O}(\Ldim^2(\calH))$               & New result in this paper (\prettyref{thm:agnostic-littlestone}). \\
                            & $\Tilde{O}(\vc(\calH){\vc^*}^2(\calH))$               & $2^{\Tilde{O}(\vc^2(\calH){\vc^*}^2(\calH))}\Ldim(\calH)$               & New result in this paper (\prettyref{thm:agnsotic-vc-littlestone}). \\ \bottomrule
\end{tabular}
\caption{We show that a hypothesis class $\calH$ is robustly learnable in the \smodel if and only if $\calH$ is \emph{online} learnable. We give upperbounds (corresponding to algorithms) in the realizable setting (\prettyref{sec:realizable-mistakeoracle}) and the agnostic setting (\prettyref{sec:agnostic-mistakeoracle}), and lower bounds on the oracle complexity in the realizable setting (\prettyref{sec:lowerbound-mistakeoracle}). Futhermore, our results show that sophisticated algorithms that leverage online learners can be favorable to more traditional online-to-batch conversion schemes in terms of their robust generalization guarantees. The $\Tilde{O}$ notation hides logarithmic factors and dependence on error $\eps$ and failure probability $\delta$, $\vc(\calH)$ and $\vc^*(\calH)$ denote the primal and dual VC dimension of $\calH$, and $\Ldim(\calH)$ denotes the Littlestone dimension of $\calH$.}

\label{tab:mistakeoracle-results}
\end{table}

\vspace{0.3\baselineskip}

To program such a perfect attack oracle, $\calU$ still has to be specified inside it. And even for simple $\calU$, a perfect attack oracle is generally intractable. Furthermore, practical attack engines used in training \citep[e.g.,~PGD][]{DBLP:conf/iclr/MadryMSTV18} are not perfect, and are not always guaranteed to find miss-classified adversarial perturbations even when they do exist. Can we still provide meaningful robustness guarantees if we only have access to an \emph{imperfect} attack oracle?

\paragraph{Access to an (imperfect) adversarial attack oracle} In this model, the learning algorithm has access to a possibly imperfect attacking algorithm $\bbA$ for $\calU$. The learning algorithm can query $\bbA$ by calling it $T$ times with queries of the form: $(g_t, (x'_t,y'_t))$, where $g_t:\calX\to \calY$ is a predictor and $(x'_t,y'_t)$ is a labeled example. The goal of the learning algorithm is to output a predictor $\hat{h}$ with small error w.r.t.~future attacks from $\bbA$,
\begin{equation}
   \label{eqn:attack-risk}
    \err_{\bbA}(\hat{h};\calD)\triangleq \Prob_{\underset{\text{randomness of }\bbA}{(x,y)\sim\calD}} \insquare{\hat{h}(\bbA(\hat{h}, (x,y))) \neq y}.
\end{equation}

In \prettyref{sec:imperfect}, we give an algorithm with sample and oracle complexity of $\Tilde{O}(\Ldim(\calH))$ that guarantees small $\err_{\bbA}$ when then attacker $\bbA$ is ``stationary'', i.e.,~$\bbA$ doesn't learn or adapt over time. 

But what happens if the adversary $\bbA$ changes over time? In the above model, the predictor $\hat{h}$ is fixed after training, and thus if the adversary $\bbA$ changes, e.g.,~by adapting to the returned predictor, or perhaps if we encounter an altogether different adversary than the one we accessed during training, this might result in a much higher error rate. Is it possible to continually adapt to changing adversaries in a meaningful way, ensuring strong robustness guarantees? 

\paragraph{Interaction with an actual attacker} 

In this online model, the learning algorithm $\calB$ can monitor the behaviour of an actual attacker $\bbA$ and adapt accordingly.  The attacker knows the current   predictor $h$ used, as well as the perturbation set $\calU$, and attempts attacks on an iid stream of samples $(x_t,y_t)$.  Whenever the attacker succeeds in finding a perturbation $z_t\in\calU(x_t)$ s.t.~$h(x_t)\neq y_t$, it scores a ``successful attack'', but the perturbation $z_t$ is revealed to learner $\calB$, who can also obtain the true label $y_t$, and learner $\calB$ can update its predictor. The goal of the learner is to bound the total number of successful attacks.

Monitoring and adapting to an attacker might sometimes be possible and appropriate, e.g.,~when attacks to predictors can be detected in hindsight and when the predictor is running on the cloud or when predictor updates can be pushed to devices, which is becoming increasingly common.  But beyond such scenarios, this online model is also useful as an analysis tool of the imperfect attack oracle model above, and our methods for the imperfect attack oracle model are based on this online model.

In \prettyref{sec:online}, we show upper bounds and lower bounds on the the number of successful attacks in terms of the Littlestone dimension $\Ldim(\calH)$, although our results leave open a possible exponential gap in the bound on the number of successful attacks (in a setting where the learner has access to infinitely many uncorrupted samples, i.e.~knows the uncorrupted source distribution).

\vspace{0.5\baselineskip}

We view the models above as examples of possible forms of accesses and models of interaction with $\calU$, and pose the question of what other reasonable models would allow generic adversarially robust learning.

\paragraph{Related Work} Most prior work on adversarial robustness has focused on methods that are tailored to specific perturbation sets $\calU$. For example, in randomized smoothing \citep{DBLP:conf/sp/LecuyerAG0J19,DBLP:conf/icml/CohenRK19,DBLP:conf/nips/SalmanLRZZBY19}, computing a prediction on a test-point $x$ requires sampling perturbations $z$ from a distribution $P$ over $\calU(x)$, and returning the most likely prediction given by a learned predictor $f:\calX\to \calY$. Distribution $P$ is chosen based on $\calU$, for example, if $\calU$ is $\ell_2$ perturbations, then $P$ is an isotropic Gaussian distribution $\calN(x,\sigma^2I)$. In addition, there are algorithms (certified defenses) that minimize some surrogate loss $\ell_\calU$ where the construction of $\ell_\calU$ depends on $\calU$ \citep[e.g.,][]{wong2018provable,raghunathan2018certified,raghunathan2018semidefinite}.

The adversarial training framework \citep{DBLP:journals/corr/GoodfellowSS14, DBLP:conf/iclr/MadryMSTV18} does not use explicit knowledge of $\calU$, but only uses an attacking algorithm (e.g.,~FGSM or PGD) implemented for the perturbation set $\calU$. However, no formal guarantees are known about adversarial training in terms of robust generalization. Specifically, it is not known whether adversarial training will yield predictors that generalize to future adversarial perturbations from $\calU$, or even generalize to specific perturbations chosen by PGD or FGSM. It has been observed that common forms of adversarial training on deep neural nets do not generalize to future attacks from PGD \citep{schmidt2018adversarially}. Our work can be seen as a theoretical study of such generic approaches, which leads to different, and considerably more sophisticated methods (yet at this stage, perhaps not easily implementable).

Towards our quest in this paper for finding the right form of access to $\calU$, we build on algorithms by \citet{pmlr-v99-montasser19a} and \citet{DBLP:conf/nips/MontasserHS20} by re-interpreting them in light of our questions, but also extending them significantly in the following ways: we avoid using a \emph{robust} empirical risk minimization $\RERM_\calU$ oracle that requires explicit knowledge of $\calU$ as used in \citet{pmlr-v99-montasser19a} and use an online learning algorithm instead, and we carry-out a technical inflation procedure of the training sample to include perturbations by utilizing a perfect attack oracle for $\calU$ without explicit knowledge of $\calU$ as was done in \citet{pmlr-v99-montasser19a,DBLP:conf/nips/MontasserHS20}. Furthermore, \citet{DBLP:conf/nips/MontasserHS20} considered robustly PAC learning $\calH$ using only black-box access to a non-robust PAC learner for $\calH$ but allowed explicit knowledge of $\calU$, their reduction makes oracle calls that depend on the bit complexity $\log\abs{\calU}$, and they show this is unavoidable. In this work, our algorithms can be viewed as black-box reductions that use an \emph{online} learner for $\calH$ (instead of just a PAC learner), furthermore, they do not require a explicit knowledge of $\calU$ but only an attack oracle for $\calU$. Our algorithms achieve the same sample complexity bound, but with number of calls to the online learner that is independent of $\log\abs{\calU}$ and only depends on the VC dimension $\vc(\calH)$.

\citet{DBLP:conf/icml/AshtianiPU20} considered a weaker form of attacking algorithms -- those that receive as input a black-box predictor -- in a $\calU$-specific learning framework and showed an upper bound of $\Tilde{O}(\vc(\calH))$ on the sample complexity of robust PAC learning when the pair $(\calH,\calU)$ admits a query efficient attacking algorithm. Their learning algorithm relies on a \emph{robust} empirical risk minimization $\RERM_\calU$ oracle that requires explicit knowledge of $\calU$. In this work, we focus on modularity and avoid using a $\RERM$ oracle for $\calH$, and use a black-box online learner $\calA$ for $\calH$. 

\citet{DBLP:conf/nips/GoldwasserKKM20} considered classifying \emph{arbitrary} test examples in a transductive selective classification setting. They gave an algorithm that takes as input: (a) training examples from a distribution $P$ over $\calX$ labeled with some unknown function $h^*$ in a class $\calH$ with finite VC dimension, and (b) a batch of arbitrary unlabeled test examples (possibly chosen by an unknown adversary), and outputs a selective predictor $\hat{f}$ -- which abstains from predicting on some examples -- that has a low rejection rate w.r.t.~$P$, and low error rate on the test examples. Selective predictor $\hat{f}$, however, can potentially abstain from classifying most test examples if they are adversarial. In this paper, we consider classifying test examples $x\sim P$ or adversarial perturbations $z\in\calU(x)$ where the perturbation set $\calU$ is unknown, and output predictors that do not abstain but always provide a classification with low error rate. We do not require unlabeled test examples, but require black-box access to an attack oracle for $\calU$.
\section{Preliminaries}

Let $\calX$ denote the instance space, $\calY=\sett{\pm1}$ denote the label space, and $\calH\subseteq \calY^\calX$ denote a hypothesis class. We denote by $\vc(\calH)$ the VC dimension of $\calH$. For a \emph{dual space} $\calG$: a set of functions $g_{x} : \calH \to \calY$ defined as $g_{x}(h) = h(x)$, for each $h \in \calH$ and each $x \in \calX$, the dual VC dimesion of $\calH$, denoted $\vc^*(\calH)$, is defined as the VC dimension of $\calG$. The dual VC dimesion is known to satisfy: $\vc^{*}(\calH) < 2^{\vc(\calH)+1}$ \citep{assouad:83}. While this exponential dependence is tight for some classes, for many natural classes, such as linear predictors and some neural networks, the primal and dual VC dimensions are equal, or at least polynomially related. The dual VC dimension is utilized in the study of adversarially robust learning \citep{pmlr-v99-montasser19a}, which we formally define next:

\begin{defn} [Robust PAC Learnability]
\label{def:robust-pac}
Learner $\calB$ $(\eps, \delta)$-robustly PAC learns $\calH \subseteq \calY^\calX$ with sample complexity $m(\eps,\delta):(0,1)^2\to \bbN$ if for any perturbation set $\calU$, any distribution $\calD$ over $\calX\times \calY$, with probability at least $1-\delta$ over $S\sim \calD^{m(\eps,\delta)}:\Risk_{\calU}(\calB(S);\calD)\leq \inf_{h\in\calH} \Risk_\calU(h;\calD) + \epsilon$.
\end{defn}

\paragraph{Online Learnability and Littlestone Dimension}

An online learning algorithm $\calA$ is a (measurable) map $(\calX \times \calY)^* \to \calY^\calX$. For a class $\calH \subseteq \calY^{\calX}$, the mistake bound of $\calA$ is the maximum possible number of mistakes algorithm $\calA$ makes on any sequence of examples labeled with some $h\in \calH$:
\begin{equation}
\label{eqn:mistakebound}
    M(\calA,\calH) := \sup_{x_1,x_2,\dots\in \calX}\sup_{h\in \calH} \sum_{t=1}^{\infty} \ind \insquare{\calA(\sett{(x_i,h(x_i))}_{i=1}^{t-1})(x_t)\neq h(x_t)}.
\end{equation}

We say that a class $\calH$ is online learnable if there exists an online learning algorithm $\calA$ such that $M(\calA,\calH)<\infty$.  A class $\calH$ is online learnable if and only if  the Littlestone dimension of $\calH$ denoted $\Ldim(\calH)$ is finite \citep{DBLP:journals/ml/Littlestone87}. Furthermore, \cite{DBLP:journals/ml/Littlestone87} proposed the Standard Optimal Algorithm (\alg{SOA}) and showed that $M(\alg{SOA},\calH)\leq \Ldim(\calH)$. We now briefly recall the definition of Littlestone dimension by introducing the notion of Littlestone trees:

\begin{defn} [Littlestone trees]
\label{def:littlestone-tree}
A Littlestone tree for $\calH$ is a complete binary tree of depth $d\leq \infty$ whose internal nodes are labeled by instances from $\calX$, and whose two edges connecting a node to its children are labeled with $+1$ and $-1$ such that every finite path emanating from the root is consistent with some concept in $\calH$. That is, a Littlestone tree is a collection$\sett{x_\bu: 0\leq k < d, \bu \in \sett{\pm1}^k} \subseteq \calX$ such that for every $\by\in\sett{\pm1}^{d}$, there exists $h\in\calH$ such that $h(x_{\by_{1:k}})=y_{k+1}$ for $0\leq k < d$. 
\end{defn}

\begin{defn}[Littlestone dimension]
\label{def:littlestone-dim}
The Littlestone dimension of $\calH$, denoted $\Ldim(\calH)$, is the largest integer $d$ such that there exists a Littlestone tree for $\calH$ of depth $d$ (see \prettyref{def:littlestone-tree}). If no such $d$ exists, then $\Ldim(\calH)$ is said to be infinite.
\end{defn}
\section{Access to a Perfect Attack Oracle}
\label{sec:mistakeoracle}
In this section, we study robust learning with algorithms that are only allowed access to a perfect attack oracle for $\calU$ at training-time. Formally,

\begin{defn} [Perfect Attack Oracle]
\label{def:mistake-oracle}
Denote by $\sfO_{\calU}$ a perfect attack oracle for $\calU$.
$\sfO_{\calU}$ takes as input a predictor $f:\calX\to \calY$ and an example $(x,y) \in \calX \times \calY$ and either: (a) asserts that $f$ is robust on $(x,y)$ (i.e.~$\forall z\in \calU(x), f(z)=y$), or (b) returns a perturbation $z\in\calU(x)$ such that $f(z)\neq y$.\footnote{To be clear, we suppose $\sfO_\calU$ acts as a function so that the $z$ it returns from calling $\sfO_\calU(g,(x,y))$ is deterministic and oblivious to the history of interactions.}
\end{defn}

In the \smodel, a learning algorithm $\calB$ takes as input iid distributed training samples $S=\sett{(x_1,y_1),\dots,(x_m,y_m)}$ drawn from an unknown distribution $\calD$ over $\calX\times \calY$, and a black-box perfect attack oracle $\sfO_\calU$. Learner $\calB$ can query $\sfO_\calU$ by calling it $T$ times with queries of the form: $(g_t, (x'_t,y'_t))$, where $g_t:\calX\to \calY$ is a predictor and $(x'_t,y'_t)$ is a labeled example. The goal of learner $\calB$ is to output a predictor $\hat{h}\in \calY^{\calX}$ with small robust risk $\Risk_\calU(\hat{h};\calD)\leq \eps$ (see \prettyref{eqn:rob-risk}). Learner $\calB$ $(\eps,\delta)$-robustly PAC learns $\calH$ in the \smodel with oracle complexity $T(\eps,\delta)$ if for any perturbation set $\calU$, learner $\calB$ $(\eps,\delta)$-robustly PAC learns $\calH$ with at most $T(\eps,\delta)$ calls to $\sfO_\calU$. 

From a practical or engineering perspective, to establish robust generalization guarantees with respect to~$\calU$ in the \smodel, it suffices to build a perfect attack oracle for $\calU$. Furthermore, to achieve robustness guarantees to multiple perturbation sets $\calU_1,\dots,\calU_k$ concurrently, which is a goal of interest in practice \citep[see e.g.,][]{DBLP:journals/corr/abs-1908-08016, DBLP:conf/nips/TramerB19, DBLP:conf/icml/MainiWK20}, it suffices to \emph{separately} build  perfect attack oracles $\sfO_{\calU_1}, \dots, \sfO_{\calU_k}$, and then implement a perfect attack oracle for the union $\cup_{i\leq k} \calU_i$ by calling each attack oracle $\sfO_{\calU_1}, \dots, \sfO_{\calU_k}$ separately. 

\begin{mdframed}[frametitle={\colorbox{white}{Questions:}},
frametitleaboveskip=-\ht\strutbox,
frametitlealignment=\center,
linecolor=Sepia!30,linewidth=1pt
outerlinewidth=1pt,
skipabove=0pt,
skipbelow=0pt]
\vskip -3mm
\centering
What hypothesis classes $\calH$ are robustly PAC learnable in the \smodel? \\
How can we learn a generic $\calH$ using such access? \\
With how many samples $m$ and oracle calls $T$?
\end{mdframed}

\paragraph{Summary of Results} We begin in \prettyref{sec:realizable-mistakeoracle} with the realizable setting under which it is assumed that there is a predictor $h^*\in \calH$ with zero robust risk, i.e.~$\inf_{h\in \calH} \Risk_\calU(h; \calD)=0$. In \prettyref{thm:realizable-littlestone-2}, we give a simple algorithm  (\prettyref{alg:robust_learner_unknown}) \alg{\cycalg} that robustly learns $\calH$ in the \smodel with sample complexity $m=O(\Ldim(\calH))$ and oracle complexity $T=O(\Ldim^2(\calH))$. In \prettyref{thm:realizable-strongust}, we give an alternative algorithm (\prettyref{alg:robust_learner_unknown}) \alg{\rlua} to robustly learn $\calH$ in the \smodel with reduced sample complexity $m=\Tilde{O}(\vc(\calH){\vc^*}^2(\calH))$ depending only the VC and dual VC dimension but at the cost of higher oracle complexity  $T\approx 2^{\Tilde{O}(\vc^2(\calH){\vc^*}^2(\calH))}\Ldim(\calH)$. Then, in \prettyref{sec:agnostic-mistakeoracle}, we extend our algorithmic results in \prettyref{thm:agnostic-littlestone} and \prettyref{thm:agnsotic-vc-littlestone} to the more general agnostic setting where we want to compete with the best attainable robust risk $\inf_{h\in\calH} \Risk_\calU(h;\calD)$. Finally, in \prettyref{sec:lowerbound-mistakeoracle}, we give a lower bound on the oracle complexity necessary to robustly learn in the \smodel. In \prettyref{cor:logloglowerbound}, we show that for any class $\calH$, the oracle complexity to robustly learn $\calH$ is at least $\Omega(\log\log \Ldim(\calH))$. Furthermore, \prettyref{cor:thresholds} gives a specific hypothesis class $\calH$ with $\vc(\calH)=O(1) \ll \Ldim(\calH)$ such that the oracle complexity to robustly learn $\calH$ is at least $\Omega(\Ldim(\calH))$. These results are summarized in \prettyref{tab:mistakeoracle-results}. 

\paragraph{Related Work}\citet{DBLP:conf/icml/MontasserGDS20} recently gave an algorithm based on the Ellipsoid method to efficiently robustly learn halfspaces (linear predictors) in the \smodel in the realizable setting, for a broad range of perturbation sets $\calU$ given access to a separation oracle for $\calU$, with oracle complexity that depends on the bit complexity. Furthermore, using a conservative online learner $\calA$ for $\calH$, \citet{DBLP:conf/nips/MontasserHS20} gave an algorithm to robustly PAC learn $\calH$ in the \smodel in the realizable setting with sample complexity $m(\eps,\delta)$ and oracle complexity $T(\eps,\delta)$ at most $O\inparen{\frac{\Ldim(\calH)\log\inparen{\Ldim(\calH)/\delta}}{\eps}}$. We consider robustly learning a general class $\calH$, and the more general agnostic setting.

\subsection{Algorithms and guarantees in the realizable setting}
\label{sec:realizable-mistakeoracle}

We begin in \prettyref{thm:realizable-littlestone-2} with a simple algorithm \alg{\cycalg} (\prettyref{alg:robust_learner_unknown}) based on an online-to-batch conversion that robustly PAC learns a class $\calH$ with sample complexity and oracle complexity depending on the Littlestone dimension $\Ldim(\calH)$. Specifically, \alg{\cycalg} (\prettyref{alg:robust_learner_unknown}) cycles an online learner $\calA$ for $\calH$ on the training set $S$ until it robustly correctly classifies all training examples. To establish a robust generalization guarantee, we show that \alg{\cycalg} (\prettyref{alg:robust_learner_unknown}) can be viewed as a stable compression scheme for the robust loss. This conversion technique and its connection to stable sample compression schemes have been recently studied in the standard $0$-$1$ loss setting \citep{DBLP:conf/colt/BousquetHMZ20}. The proof is provided in \prettyref{app:appendix-a}. 

\begin{thm}
\label{thm:realizable-littlestone-2}
For any class $\calH$, \alg{\cycalg} (\prettyref{alg:robust_learner_unknown}) robustly PAC learns $\calH$ w.r.t.~any $\calU$ with:
\begin{enumerate}
    \item Sample complexity $m(\eps,\delta)= O\inparen{\frac{\Ldim(\calH)+\log(1/\delta)}{\eps}}$.
    \item Oracle complexity $T(\eps,\delta)=m(\eps,\delta)\Ldim(\calH)$. 
\end{enumerate}
Furthermore, the output of \alg{\cycalg} achieves zero robust loss on the training sample. 
\end{thm}

\alg{\cycalg} (\prettyref{alg:robust_learner_unknown}) robustly PAC learns $\calH$ in the \smodel with sample complexity and oracle complexity both depending on the Littlestone dimension $\Ldim(\calH)$. But are there robust learning algorithms with better sample complexity and/or oracle complexity? At least with explicit knowledge of $\calU$, we know that we can robustly PAC learn $\calH$ with $\Tilde{O}(\vc(\calH)\vc^*(\calH))$ sample complexity \citep{pmlr-v99-montasser19a} which is much smaller than $\Ldim(\calH)$ for many natural classes (e.g., halfspaces). Can we obtain a similar sample complexity bound in the \smodel, where explicit knowledge of $\calU$ is not allowed? We prove \emph{yes} in \prettyref{thm:realizable-strongust}. Specifically, we give an algorithm that can robustly PAC learn $\calH$ in \smodel with sample complexity $\Tilde{O}(\vc(\calH){\vc^*}^2(\calH))$  independent of $\Ldim(\calH)$. 

\begin{thm}
\label{thm:realizable-strongust}
For any class $\calH$ with $\vc(\calH)=d$ and $\vc^*(\calH)=d^*$, there exists a learning algorithm $\Tilde{\calB}$ that robustly PAC learns $\calH$ w.r.t any $\calU$ with:
\begin{enumerate}
    \item Sample Complexity $m(\eps,\delta)= O\inparen{\frac{d{d^*}^2\log^2d^*}{\eps}\log^2\inparen{\frac{d{d^*}^2\log^2d^*}{\eps}} + \frac{1}{\eps}\log(\frac{2}{\delta})}.$
    \item Oracle Complexity $T_{\rm RE}(\eps,\delta)=\inparen{2^{O(d^2{d^*}^2\log^2d^*)}\Ldim(\calH)+m(\eps,\delta)}\inparen{\log\inparen{m(\eps,\delta)}\inparen{\log\inparen{\frac{\log(m(\eps,\delta))}{\delta}}}}$.
\end{enumerate}
\end{thm}

\begin{algorithm2e}
\caption{Robust Learner against Unknown Adversaries (\alg{\rlua})}\label{alg:robust_learner_unknown}
\SetKwInput{KwInput}{Input}                
\SetKwInput{KwOutput}{Output}              
\SetKwFunction{FMain}{CycleRobust}
\SetKwFunction{FDisc}{Discretizer}
\DontPrintSemicolon
    \KwInput{Training dataset $S=\sett{(x_1,y_1),\dots, (x_m,y_m)}$, black-box online learner $\calA$ for $\calH$, black-box perfect attack oracle $\sfO_\calU$.}
  \BlankLine
  Set $n =O(\vc(\calA))$.\;
  Foreach $L\subset S$ such that $\abs{L}=n$, run \FMain on $(L, \calA,\sfO_\calU)$, and denote by $\hat{\calH}$ the resulting set of predictors.\;
  Call \FDisc on $(S, \hat{\calH}, \sfO_\calU)$, and denote by $\hat{S}_\calU$ its output.\;
  Initialize $D_1$ to be uniform over $\hat{S}_\calU$, and set $T=O(\log|S_\calU|)$.\;
  \For{$1 \leq t\leq T$}{
    Sample $S'\sim D_t^{n}$, and project $S'$ to dataset $L_t\subseteq S$ by replacing each perturbation $z$ with its corresponding example $x$.\;
    Call \FMain on $(L_t, \calA,\sfO_\calU)$, and denote by $f_t$ its output predictor.\;
    Compute a new distribution $D_{t+1}$ by applying the following update for each $(z,y)\in \hat{S}_\calU$:
            \[ 
                D_{t+1}(\sett{(z,y)}) = \frac{D_t(\sett{(z,y)})}{Z_t} \times \begin{cases} 
                                                              e^{-2\alpha}& \text{if }f_t(z)=y\\
                                                              1 &\text{otherwise}
                                                           \end{cases}
            \]
            where $Z_t$ is a normalization factor and $\alpha$ is set as in Lemma~\ref{lem:alphaboost}.\;
  }
  Sample $N=O(\vc^*(\calA))$ i.i.d. indices $i_{1},\ldots,i_{N} \sim {\rm Uniform}(\{1,\ldots,T\})$.
  \\~~~~(repeat previous step until $g=\MAJ(f_{i_1},\ldots,f_{i_{N}})$ satisfies $\Risk_{\calU}(g;S)=0$)\;
\KwOutput{A majority-vote $\MAJ(f_{i_1},\dots, f_{i_{N}})$ predictor.}
  \BlankLine
  \SetKwProg{Fn}{}{:}{\KwRet}
  \Fn{\FMain{Dataset $L$, Learner $\calA$, Oracle $\sfO_\calU$}}{
        Initialize $Z=\sett{}$, and initialize $\hat{h}=\calA(Z)$.\;
        Set $\texttt{FullRobustPass} = \texttt{False}$.\;
        \While{\texttt{FullRobustPass} is \texttt{False}}{
            Set $\texttt{FullRobustPass} = \texttt{True}$.\;
            \For{$1 \leq i\leq m$}{
            Certify the robustness of $\hat{h}$ on $(x_i,y_i)$ by sending the query $(\hat{h}, (x_i,y_i))$ to the perfect attack oracle $\sfO_\calU$.\;
            \If{$\hat{h}$ is not robustly correct on $(x_i,y_i)$}{
                Let $z$ be the perturbation returned by $\sfO_\calU$ where $\hat{h}(z)\neq y_i$.\;
                Add $(z,y_i)$ to the set $Z$.\;
                Update $\hat{h}$ by running $\calA$ on example $(z,y_i)$, or equivalently, set $\hat{h} = \calA(Z)$.\;
                Set $\texttt{FullRobustPass} = \texttt{False}$.\;
                }
            }
        }
        \KwRet Predictor $\hat{h}$.\;}
  \BlankLine
  \SetKwProg{Fn}{}{:}{\KwRet}
  \Fn{\FDisc{Dataset $S$, Predictors $\hat{\calH}$, Oracle $\sfO_\calU$}}{
    Initialize 
    \For{$(x,y)\in S$}{
        Initialize $P=\sett{(x,y)}$.\;
        Let $f^y_P:\calX \to \calY$ be a predictor of the form:\\
        ~~~~$f^y_P(x') = y~\text{if and only if}~\inparen{\exists_{(z,y)\in P}}\inparen{\forall_{h \in \hat{\calH}}}\ind_{[ h(z) \neq y ]}=\ind_{[ h(x') \neq y ]}$.\;
        Send the query $(f^y_P, (x,y))$ to the perfect attack oracle $\sfO_\calU$.\;  
        \While{$f^y_P$ is not robustly correct on $(x,y)$}{
            Let $z$ be the perturbation returned by $\sfO_\calU$ where $f^y_P(z)\neq y$.\;
            Append $(z,y)$ to the set $P$.\;
            Send an updated query $(f^y_P, (x,y))$ to the perfect attack oracle $\sfO_\calU$.\;
        }
    }
    
  }
\end{algorithm2e}

The full proof is deferred to \prettyref{app:appendix-vc-upperbound}, but we briefly describe the main building blocks of this result. We will adapt an algorithm due to \citet{pmlr-v99-montasser19a} and establish a robust generalization guarantee that depends only on  $\vc(\calH)$ and $\vc^*(\calH)$. In particular, the learning algorithm of \cite{pmlr-v99-montasser19a} required explicit knowledge of $\calU$, this knowledge was used to implement a $\RERM_\calU$ oracle for $\calH$, and for a sample inflation and discretization step which is crucial to establish robust generalization based on sample compression. As explicit knowledge is not allowed in the \smodel, we show that we can avoid these limitations and use only queries to $\sfO_\calU$. Specifically, observe that \alg{\cycalg} (\prettyref{alg:robust_learner_unknown}) implements a $\RERM_\calU$ oracle for $\calH$ using only black-box access to $\sfO_\calU$, since by \prettyref{thm:realizable-littlestone-2} the output of \alg{\cycalg} achieves zero robust loss on its input dataset $S$. Similarly, the discretization step can be carried using only queries to $\sfO_\calU$, by constructing queries using the output predictors of \alg{\cycalg} to force the oracle to reveal perturbations of the empirical sample $S$, we leave further details to the proof. While this suffices to establish a result of robust PAC learning in the \smodel with sample complexity completely independent of $\Ldim(\calH)$ (\prettyref{thm:realizable-weak-robust} and its proof in \prettyref{app:appendix-vc-upperbound}), we can further improve the dependence on $\eps$ and $\delta$ in the oracle complexity. To this end, we treat the algorithm (\prettyref{alg:robust_learner_unknown}) \alg{\rlua} from \prettyref{thm:realizable-weak-robust} as a \emph{weak} robust learner with fixed $\eps_0$ and $\delta_0$ and boost its robust error guarantee to improve the oracle complexity and obtain the result in \prettyref{thm:realizable-strongust}. 
\subsection{Algorithms and guarantees in the agnostic setting}
\label{sec:agnostic-mistakeoracle}

We now consider the more general agnostic setting where we want to compete with the best attainable robust risk $\inf_{h \in \calH} \Risk_\calU(h;\calD)$. Mirroring the results from the realizable section, we begin in \prettyref{thm:agnostic-littlestone} with a simple algorithm that can only guarantee a robust error at most $2\inf_{h \in \calH} \Risk_\calU(h;\calD)+\eps$ with sample and oracle complexity depending on the Littlestone dimension $\Ldim(\calH)$. Then, in \prettyref{thm:agnsotic-vc-littlestone}, we give a reduction to the realizable setting of \prettyref{thm:realizable-strongust}, that agnostically robustly PAC learns $\calH$ in the \smodel with sample complexity depending only on the $\vc(\calH)$ and $\vc^*(\calH)$. 

\begin{thm}
\label{thm:agnostic-littlestone}
For any class $\calH$, \alg{Weighted Majority} (\prettyref{alg:weighted-majority}) guarantees that for any perturbation set $\calU$ and any distribution $\calD$ over $\calX \times \calY$, with sample complexity $m(\eps,\delta)=O\inparen{\frac{\Ldim(\calH)\log(1/\eps)+\log(1/\delta)}{\eps^2}}$ and oracle complexity $T(\eps,\delta)=O(m(\eps,\delta)^2)$, with probability at least $1-\delta$ over $S\sim \calD^{m(\eps,\delta)}$,
\[ \Risk_{\calU}(\alg{WM}(S,\sfO_\calU);\calD) \leq 2\inf_{h\in\calH}\Risk_\calU(h;\calD)+\eps.\]
\end{thm}

We briefly describe the main ingredients of this result. First, in \prettyref{lem:agnostic-finite}, we show that for any class $\calH$ with finite cardinality, a variant of the \alg{Weighted Majority} algorithm \citep{DBLP:journals/iandc/LittlestoneW94} presented in  (\prettyref{alg:weighted-majority}) has a regret guarantee with respect to the robust loss. Then, in \prettyref{lem:agnostic-littlestone}, we extend this regret guarantee for infinite $\calH$ using a technique due to \cite{DBLP:conf/colt/Ben-DavidPS09} for agnostic online learning. Finally, we apply a standard online-to-batch conversion \citep{DBLP:journals/tit/Cesa-BianchiCG04} to convert the regret guarantee to a robust generalization guarantee. These helper lemmas and proofs are deferred to \prettyref{app:appendix-b}.

Similarly to the realizable setting, we can establish an upper bound with sample complexity independent of $\Ldim(\calH)$. This is achieved via a reduction to the realizable setting \prettyref{thm:realizable-strongust}, following an argument of \citet{david:16,pmlr-v99-montasser19a}. The proof is deferred to \prettyref{app:appendix-b}.

\begin{thm} [Reduction to Realizable Setting]
\label{thm:agnsotic-vc-littlestone}
For any class $\calH$ with $\vc(\calH)=d$ and $\vc^*(\calH)=d^*$, there is a learning algorithm $\Tilde{\calB}$ that robustly \emph{agnostically} PAC learns $\calH$ w.r.t any $\calU$ with:
\begin{enumerate}
    \item Sample Complexity $m(\eps,\delta)= O\inparen{\frac{d{d^*}^2\log^2d^*}{\eps^2}\log^2\inparen{\frac{d{d^*}^2\log^2d^*}{\eps}} + \frac{1}{\eps^2}\log(\frac{2}{\delta})}$.
    \item Oracle Complexity $T(\eps,\delta)=2^{m(\eps,\delta)}\Ldim(\calH)+T_{\rm RE}(\eps,\delta)$.
\end{enumerate}
\end{thm}
\subsection{Lowerbound on Oracle Complexity}
\label{sec:lowerbound-mistakeoracle}

In \prettyref{sec:realizable-mistakeoracle} and \prettyref{sec:agnostic-mistakeoracle}, we have shown that it is possible to robustly PAC learn in the \smodel with sample complexity that is completely \emph{independent} of the Littlestone dimension, and with oracle complexity that \emph{depends} on the Littlestone dimension. It is natural to ask whether the oracle complexity can be improved. Perhaps, we can avoid dependence on Littlestone dimension altogether? In \prettyref{thm:lowerbound-oraclecomplexity}, we prove that the answer is \emph{no}. 

Specifically, we will first establish a lower bound in terms of another complexity measure known as the Threshold dimension of $\calH$, denoted by $\Tdim(\calH)$. Informally, $\Tdim(\calH)$ is the largest number of thresholds that can be embedded in class $\calH$ (see \prettyref{def:threshold-dim} below). Importantly, the Threshold dimension of $\calH$ is related to the Littlestone dimension of $\calH$ and is known to satisfy: $\floor{ \log_2\Ldim(\calH)} \leq \Tdim(\calH) \leq 2^{\Ldim(\calH)}$ \citep{DBLP:books/daglib/0067545, DBLP:books/daglib/0030198,DBLP:conf/stoc/AlonLMM19}. This relationship was recently used to establish that \emph{private} PAC learnability implies online learnability \citep{DBLP:conf/stoc/AlonLMM19}.

\begin{defn} [Threshold Dimension]
\label{def:threshold-dim}
We say that a class $\calH\subseteq \calY^\calX$ contains $k$ thresholds if $\exists x_1,\dots, x_k \in \calX$ and $\exists h_1,\dots, h_k \in \calH$ such that $h_i(x_j)=+1$ if and only if $j\leq i,\forall i,j\leq k$. The Threshold dimension of $\calH$, $\Tdim(\calH)$, is the largest integer $k$ such that $\calH$ contains $k$ thresholds.
\end{defn}

\begin{thm}
\label{thm:lowerbound-oraclecomplexity}
For any class $\calH$, there exists a distribution $\calD$ over $\calX\times \calY$, such that for any learner $\calB$, there exists a perturbation set $\calU: \calX \to 2^\calX$ where $\inf_{h\in \calH} \Risk_\calU(h;\calD)=0$ and a perfect attack oracle $\sfO_\calU$ such that $\calB$ needs to make $\frac{\log_2{\inparen{\Tdim(\calH)-1}}}{2}$ oracle queries to $\sfO_\calU$ to robustly learn $\calD$.
\end{thm}

The full proof is deferred to \prettyref{app:appendix-lowerbound}, but we will provide some intuition behind the proof. The main idea is to use $h_1,\dots, h_{\Tdim(\calH)}$ thresholds to construct $\calU_1,\dots,\calU_{\Tdim(\calH)}$ perturbation sets. We will setup a single source distribution $\calD$ that is known to the learner, but choose a random perturbation set from $\calU_1,\dots,\calU_{\Tdim(\calH)}$. In order for the learner to \emph{robustly} learn $\calD$, it needs to figure out which perturbation set is picked, and that requires $\Omega(\log\Tdim(\calH))$ queries to the oracle $\sfO_\calU$. Since $\Tdim(\calH) \geq \floor{ \log_2\Ldim(\calH)}$, \prettyref{thm:lowerbound-oraclecomplexity} implies the following corollaries.

\begin{cor}
\label{cor:logloglowerbound}
For any class $\calH$, the oracle complexity to robustly learn $\calH$ in the \smodel is at least $\Omega(\log\log\Ldim(\calH))$.
\end{cor}

\begin{cor}
\label{cor:thresholds}
For any $n\in\bbN$, the class $\calH_n$ consisting of $n$ thresholds satisfies $\Ldim(\calH_n) = \log_2 \Tdim(\calH_n) = \log_2 n$\footnote{We learned about this fact from the following talk: \url{https://youtu.be/NPpPiWYcmPk}} and $\vc(\calH_n)=O(1)$. Thus, the oracle complexity to robustly learn $\calH_n$ in the \smodel is $\Omega(\Ldim(\calH_n))$.
\end{cor}

A couple of remarks are in order. First, the lower bound of $\Omega(\log\log\Ldim(\calH))$ applies to any hypothesis class $\calH$, but a stronger lower bound for the special case of thresholds can be shown where $\Omega(\Ldim(\calH))$ oracle queries are needed. Second, observe that these lower bounds apply to learning algorithms that know the distribution $\calD$, and so even with infinite sample complexity, it is not possible to have oracle complexity independent of Littlestone dimension.

\subsection{Gaps and Open Questions}

We have established that $\calH$ is robustly PAC learnable in the \smodel if and only if $\calH$ is online learnable. We provided a simple online-to-batch conversion scheme \alg{\cycalg} (\prettyref{alg:robust_learner_unknown}) with sample and oracle complexity scaling with $\Ldim(\calH)$. Then, with a more sophisticated algorithm, \alg{\rlua} (\prettyref{alg:robust_learner_unknown}), we get an improved sample complexity depending only on $\vc(\calH)$ and $\vc^*(\calH)$, but at the expense of higher oracle complexity with an exponential dependence on $\vc(\calH)$ and $\vc^*(\calH)$ and linear dependence on $\Ldim(\calH)$. We also showed that for any class $\calH$, an oracle complexity of $\log\log\Ldim(\calH)$ is unavoidable, and furthermore, exhibit a class $\calH$ with $\vc(\calH)=O(1)$ and $\Ldim(\calH)\gg \vc(\calH)$ where an oracle complexity of $\Omega(\Ldim(\calH))$ is unavoidable.

An interesting direction is to improve the oracle complexity to perhaps a polynomial dependence ${\rm poly}(\vc(\calH),\vc^*(\calH))\Ldim(\calH)$, or more ambitiously ${\rm poly}(\vc(\calH))\Ldim(\calH)$. It would also be interesting to establish a finer characterization for the oracle complexity that is adaptive to the perturbation sets $\calU$, perhaps depending on some notion measuring the complexity of $\calU$. Also, can we strengthen the lower bound and show that for any $\calH$, $\Omega(\Ldim(\calH))$ oracle complexity is necessary to robustly learn $\calH$ or is there another complexity measure that tightly characterizes the oracle complexity?
\section{Bounding the number of successful attacks}
\label{sec:online}

In \prettyref{sec:mistakeoracle}, we considered having access to a \emph{perfect} attack oracle $\sfO_\calU$. But in many settings, our practical attack oracle attack engines, e.g.,~PGD \citep{DBLP:conf/iclr/MadryMSTV18}, are not perfect---they might not always find miss-classified adversarsial perturbations even when they do exist. Also, the perturbation set $\calU$ might be fairly abstract, like ``images indistinguishable to the human eye'', and so there isn't really a perfect attack oracle, but rather just approximations to it. Can we still provide meaningful robustness guarantees even with \emph{imperfect} attackers? 

In this section, we introduce a model where we consider working with an actual adversary or attack algorithm that is possibly imperfect, and the goal is to bound the number of successful attacks. In this model, a learning algorithm $\calB$ first receives as input iid distributed training samples $S=\sett{(x_1,y_1),\dots,(x_m,y_m)}$ drawn from an unknown distribution $\calD$ over $\calX\times \calY$. Then, the learning algorithm $\calB$ makes predictions on examples $z_t\in\calU(x'_t)$ where $z_t$ is an adversarial perturbation chosen by an adversary $\bbA$, and $(x'_t,y'_t)$ is an iid sample drawn from $\calD$. The adversary $\bbA$ has access to the random sample $(x'_t,y'_t)$ and the predictor used by learner $\calB$, $h_t = \calB(S\cup\{(z_j,y_j)_{j=1}^{t-1}\})$, but learner $\calB$ only sees the perturbation $z_t$. After learner $\calB$ makes its prediction $\hat{y}_t= h_t(z_t)$, the true label $y_t$ is revealed to $\calB$. The goal is to bound the number of successful adversarial attacks where $\hat{y}_t\neq y_t$. For a class $\calH$ and a learner $\calB$, the maximum number of successful attacks caused by an adversary $\bbA$ w.r.t.~$\calU$ on a distribution $\calD$ satisfying $\inf_{h\in\calH}\Risk_\calU(h;\calD)=0$ is defined as
\begin{equation}
\label{eqn:robustmistakebound}
    M_{\calU, \bbA}(\calB,\calH;\calD) := \sum_{t=1}^{\infty} \ind \insquare{\calB(S\cup\sett{(z_i,y_i)}_{i=1}^{t-1})(z_t)\neq y_t},
\end{equation}
where $z_t = \bbA(\calB(S\cup\{(z_i,y_i)\}_{i=1}^{t-1}), (x'_t,y'_t))$ and $\sett{(x'_t,y'_t)}^{\infty}_{t=1}$ are iid samples from $\calD$.

\begin{mdframed}[frametitle={\colorbox{white}{Questions:}},
frametitleaboveskip=-\ht\strutbox,
frametitlealignment=\center,
linecolor=Sepia!30,linewidth=1pt
outerlinewidth=1pt,
skipabove=0pt,
skipbelow=0pt]
\vskip -3mm
\centering
Can we obtain upper bounds and lower bounds on the maximum number of successful attacks for generic classes $\calH$? Can additional training samples from $\calD$ help?
\end{mdframed}

First, we show that we can upper bound the maximum number of successful attacks on any online learner $\calB$ for $\calH$ by the online mistake bound of $\calB$. 

\begin{thm} [Upper Bound]
\label{thm:online-upperbound}
For any class $\calH$ and any online learner $\calB$, for any perturbation set $\calU$, adversary $\bbA$, and distribution $\calD$, $M_{\calU,\bbA}(\calB,\calH;\calD)\leq M(\calB,\calH)$. In particular, the Standard Optimal Algorithm (\alg{SOA}) has an attack bound of at most $\Ldim(\calH)$.
\end{thm}

\begin{proof}
The proof follows directly from the definition of the online mistake bound (see \prettyref{eqn:mistakebound}) and the online attack bound (see \prettyref{eqn:robustmistakebound}). 
\end{proof}

Is this the best achievable upper bound on the number of successful attacks? Perhaps there are learning algorithms with an attack bound that is much smaller than $\Ldim(\calH)$, maybe an attack bound that scales with $\vc(\calH)$? We next answer this question in the negative. Using the same the lower bound construction from \prettyref{thm:lowerbound-oraclecomplexity} in \prettyref{sec:lowerbound-mistakeoracle}, we first establish a lower bound on the number of the successful attacks based on the Threshold dimension of $\calH$ (see \prettyref{def:threshold-dim}) (proof deferred to \prettyref{app:appendix-c}). We then utilize the relationship $\Tdim(\calH) \geq \floor{ \log_2\Ldim(\calH)}$ to get the corollaries.

\begin{thm} [Lower Bound]
\label{thm:online-lowerbound}
For any class $\calH$, there exists a distribution $\calD$ over $\calX \times \calY$, such that for any learner $\calB$, there is a perturbation set $\calU$ where $\inf_{h \in \calH} \Risk_{\calU}(h;\calD)=0$ and an adversary $\bbA$ that makes at least $\frac{\log_2{\inparen{\Tdim(\calH)-1}}}{2}$ successful attacks on learner $\calB$.
\end{thm}

\begin{cor}
\label{cor:loglog-online}
For any class $\calH$, there is a distribution $\calD$, such that for any learner $\calB$, there is a perturbation set $\calU$ and adversary $\bbA$ such that $M_{\calU,\bbA}(\calB,\calH;\calD) \geq \Omega(\log\log \Ldim(\calH)).$
\end{cor}

\begin{cor}
\label{cor:thresholds-online}
For any $n\in\bbN$, the class $\calH_n$ consisting of $n$ thresholds satisfies $\Ldim(\calH_n) = \log_2 \Tdim(\calH_n) = \log_2 n$ and $\vc(\calH_n)=O(1)$. Thus, $\exists_\calD\forall_\calB\exists_{\calU,\bbA} M_{\calU,\bbA}(\calB,\calH_n;\calD) \geq \Omega(\Ldim(\calH_n))$.
\end{cor}

We remark that these lower bounds hold even for learning algorithms $\calB$ that perfectly know the source distribution $\calD$. For the class $\calH_n$ of $n$ thresholds, we cannot expect a learning algorithm that leverages extra training data that avoids the $\Omega(\Ldim(\calH_n))$ lower bound. But it might be that for other classes $\calH$, additional training data might help reduce the attack bound to $\log\Ldim(\calH)$ or $\log\log \Ldim(\calH)$.

\paragraph{Gaps and Open Questions} We have only considered the realizable setting, where there is a predictor $h\in\calH$ that is perfectly robust to the attacker $\bbA$. It would be interesting to extend the guarantees to the agnostic setting. Can we strengthen the lower bound and show that for any $\calH$, an attack bound of $\Omega(\Ldim(\calH))$ is unavoidable or is there another complexity measure that tightly characterizes the attack bound? Are there examples of classes $\calH$ where collecting additional samples from $\calD$ helps reduce the number of successful attacks?

\section{Robust generalization to imperfect attack algorithms}
\label{sec:imperfect}

In \prettyref{sec:online}, given an online learning algorithm, we can guarantee a finite number of successful attacks from any attacking algorithm even if it was imperfect. But what if we want to work in a more traditional train-then-ship approach, where we first ensure adversarial robustness without releasing anything, and only then release? Can we provide any robust generalization guarantees when we only have access to an \emph{imperfect} attacking algorithm such as PGD \citep{DBLP:conf/iclr/MadryMSTV18} at training-time?

In this model, a learning algorithm $\calB$ takes as input a black-box (possibly imperfect) attacker $\bbA$, and iid distributed training samples $S=\sett{(x_1,y_1),\dots,(x_m,y_m)}$ from an unknown distribution $\calD$ over $\calX\times \calY$. Learner $\calB$ can query $\bbA$ by calling it $T$ times with queries of the form: $(g_t, (x'_t,y'_t))$, where $g_t:\calX\to \calY$ is a predictor and $(x'_t,y'_t)$ is a labeled example. The goal of learner $\calB$ is to output a predictor $\hat{h}\in \calY^{\calX}$ with small error w.r.t.~future attacks from $\bbA$, $\err_\bbA(\hat{h};\calD)\leq \eps$ (see \prettyref{eqn:attack-risk}). 

\begin{defn}[Robust PAC Learnability with Imperfect Attackers]
\label{def:attack-gen}
Learner $\calB$ $(\eps, \delta)$-robustly PAC learns $\calH \subseteq \calY^\calX$ with sample complexity $m(\eps,\delta):(0,1)^2\to \bbN$ and oracle complexity $T(\eps,\delta):(0,1)^2\to \bbN$ if for any (possibly randomized and imperfect) attacker $\bbA:\calY^\calX \times (\calX\times\calY) \to \calX$, any distribution $\calD$ over $\calX\times \calY$, with at most $T(\eps,\delta)$ oracle calls to $\bbA$ and with probability at least $1-\delta$ over $S\sim \calD^{m(\eps,\delta)}: \err_{\bbA}(\calB(S,\bbA)) \leq \inf_{h\in\calH} \err_{\bbA}(h) + \epsilon$.
\end{defn}

In this model, access to an imperfect attacker $\bbA$ at training-time ensures robust generalization to this \emph{specific} attacker $\bbA$ at test-time. This is a different guarantee from robust generalization to a perturbation set $\calU$, because it might well be that that there is a stronger attack algorithm $\bbA'$ than $\bbA$ such that $\err_{\bbA'}(\hat{h};\calD)\gg \err_{\bbA}(\hat{h};\calD)$. Furthermore, since the ``strength'' of the attack algorithm $\bbA$ is a function of the predictor $\hat{h}$ it is attacking, establishing a generalization guarantee w.r.t.~$\bbA$ is not immediate, and does not follow for example from our results in \prettyref{sec:mistakeoracle}.

We relate robust learnability in this model to the online model from \prettyref{sec:online}. In \prettyref{thm:attack-generalization}, we observe that we can apply a standard online-to-batch conversion based on the longest survivor technique \citep{10008965845} to establish generalization guarantees with respect to future attacks made by $\bbA$. Specifically, we simply output a predictor $\hat{h}$ that has survived a sufficient number of attacks from $\bbA$. The full algorithm and proof are presented in \prettyref{app:appendix-d}.

\begin{thm}
\label{thm:attack-generalization}f
For any class $\calH$, \prettyref{alg:robust-learner-attack} robustly PAC learns $\calH$ w.r.t.~any (possibly randomized and imperfect) attacker $\bbA$ and any distribution $\calD$ such that $\inf_{h\in\calH}\err_{\bbA}(h;\calD)=0$, with sample complexity $m(\eps,\delta)$ and oracle complexity $T(\eps,\delta)$ at most $O\inparen{\frac{\Ldim(\calH)\log\inparen{\Ldim(\calH)/\delta}}{\eps}}$.
\end{thm}

\paragraph{Gaps and Open Questions} Currently we only provide generalization guarantees in the realizable setting. It would be interesting to extend our guarantees to the agnostic setting. Are there algorithms with better sample complexity perhaps depending only on the VC dimension? We established such a result in \prettyref{sec:mistakeoracle} with access to a perfect attack oracle, but the same approach does not go through when using an imperfect attacker. What about better oracle complexity? Can we obtain similar generalization guarantees for adaptive attacking algorithms that change over time? Can we obtain generalization guarantees against a family of attacking algorithms (e.g.,~first order attacks)?

\section{Discussion}
\label{sec:discussion}

In this paper, we consider robust learning with respect to \emph{unknown} perturbation sets $\calU$. We initiate the quest to find the ``right'' model of access to $\calU$ by considering different forms of access and studying the robustness guarantees achievable in each case. One of the main takeaways from this work is that we need to be mindful about what form of access to $\calU$ we are assuming, because the guarantees that can be achieved can be different. So knowledge about $\calU$ should not be thought of as a free resource, but rather we should quantify the complexity of the information we are asking about $\calU$.

In some ways, adversarial learning is an arms race, and \citet{DBLP:conf/icml/AthalyeC018} have illustrated that predictors trained to be secure against a specific attack, might be easily defeated by a different attack.  Our \emph{imperfect} attack oracle model in \prettyref{sec:imperfect} certainly suffers from this problem.  But, it can also be viewed as taking a step towards addressing it, as it provides a generic way of turning any new attack into a defence, and thus defending against it, and since this is done in a black-box manner, could at the very least hasten the development time needed to defend against a new attack. The online model in \prettyref{sec:online} in a sense does so explicitly, and can indeed handle arbitrary new attacks.

In \prettyref{sec:mistakeoracle} we establish robust generalization guarantees w.r.t.~any attacking algorithm $\bbA$ for $\calU$, but it requires a perfect adversarial attack oracle for $\calU$: $\sfO_\calU$, and in \prettyref{sec:imperfect}, we establish a generalization guarantee w.r.t.~a \emph{specific} attack algorithm $\bbA$ when given black-box access to $\bbA$ at training-time. These are in a sense two opposing ends of the spectrum. Are there other interesting models that provide weaker access than a perfect attack oracle $\sfO_\calU$, but also provide a stronger guarantee than that of generalization to a particular attack? For example, under what conditions, can we generalize to a test-time attacker $\bbA_{\rm test}$ that is different from the attacker $\bbA_{\rm train}$ used at training-time. What if we are interested in providing guarantees to a family of test-time attackers (e.g. first-order algorithms), what form of access would be sufficient and necessary? 

Under explicit knowledge of the perturbation set $\calU$, \citet{pmlr-v99-montasser19a} showed that we can robustly learn any hypothesis class $\calH$ that is PAC learnable, i.e.,~finite $\vc(\calH)$. Given only a perfect attack oracle to $\calU$, we show in this paper that we can robustly learn any hypothesis class $\calH$ that is online learnable, i.e.,~finite $\Ldim(\calH)$, and we give lower bounds showing that online learnability is necessary. Are there other models of access to $\calU$ that would allow us to robustly learn a broader family of hypothesis classes beyond those that are online learnable? 

Our approach to robust learning in this work is modular, in particular, the \emph{perfect} and \emph{imperfect} attack oracles that we consider are independent of the hypothesis class $\calH$ since they just receive a predictor $g:\calX\to\calY$ as input. But how is $g$ provided? And do we expect the oracles to accept as input any predictor $g$ regardless of its complexity? A more careful look reveals that for the simple algorithm \alg{\cycalg} (\prettyref{alg:robust_learner_unknown}) it suffices for the perfect attack oracle $\sfO_\calU$ to accept only predictors $h \in \calH$. This seems like it creates a dependency and breaks the modularity, but it does not, since, e.g., the oracle might be implemented in terms of a much larger and more generic class, such as neural nets with any architecture, as opposed to the specific architecture we are trying to learn, and which the oracle need not be aware of. But the more sophisticated algorithm \alg{\rlua} (\prettyref{alg:robust_learner_unknown}) requires calling the oracle on predictors outside $\calH$ and so we do need the oracle to accept arbitrary predictors, or at least predictors from a much broader class than $\calH$.  How can this be translated to a computational rather than purely mathematical framework, and implemented in practice?

Are there sensible but generic assumptions on the perturbation set $\calU$ that can lead to improved guarantees? Either assumptions that are on $\calU$ separate from the class $\calH$, i.e.,~that hold even if $\calH$ is applied to one relabeling or permutation of $\calX$ and $\calU$ is applied to a different relabeling, or that rely on simple and generic relationships between $\calU$ and $\calH$.

\acks{This work was supported in part by DARPA under cooperative agreement HR00112020003\footnote{The views expressed in this work do not necessarily reflect the position or the policy of the Government and no official endorsement should be inferred. Approved for public release; distribution is unlimited.} and NSF BIGDATA award 1546500. Part of this work was done at the IDEAL Fall 2020 special quarter on {\it Theory of Deep Learning} funded by NSF TRIPOD award 1934843.}

\bibliography{learning}

\begin{thebibliography}{41}
\providecommand{\natexlab}[1]{#1}
\providecommand{\url}[1]{\texttt{#1}}
\expandafter\ifx\csname urlstyle\endcsname\relax
  \providecommand{\doi}[1]{doi: #1}\else
  \providecommand{\doi}{doi: \begingroup \urlstyle{rm}\Url}\fi

\bibitem[Alon et~al.(2019)Alon, Livni, Malliaris, and
  Moran]{DBLP:conf/stoc/AlonLMM19}
Noga Alon, Roi Livni, Maryanthe Malliaris, and Shay Moran.
\newblock Private {PAC} learning implies finite littlestone dimension.
\newblock In Moses Charikar and Edith Cohen, editors, \emph{Proceedings of the
  51st Annual {ACM} {SIGACT} Symposium on Theory of Computing, {STOC} 2019,
  Phoenix, AZ, USA, June 23-26, 2019}, pages 852--860. {ACM}, 2019.
\newblock \doi{10.1145/3313276.3316312}.
\newblock URL \url{https://doi.org/10.1145/3313276.3316312}.

\bibitem[Ashtiani et~al.(2020)Ashtiani, Pathak, and
  Urner]{DBLP:conf/icml/AshtianiPU20}
Hassan Ashtiani, Vinayak Pathak, and Ruth Urner.
\newblock Black-box certification and learning under adversarial perturbations.
\newblock In \emph{Proceedings of the 37th International Conference on Machine
  Learning, {ICML} 2020, 13-18 July 2020, Virtual Event}, volume 119 of
  \emph{Proceedings of Machine Learning Research}, pages 388--398. {PMLR},
  2020.
\newblock URL \url{http://proceedings.mlr.press/v119/ashtiani20a.html}.

\bibitem[Assouad(1983)]{assouad:83}
P.~Assouad.
\newblock Densit\'e et dimension.
\newblock \emph{Annales de l'Institut Fourier (Grenoble)}, 33\penalty0
  (3):\penalty0 233--282, 1983.

\bibitem[Athalye et~al.(2018)Athalye, Carlini, and
  Wagner]{DBLP:conf/icml/AthalyeC018}
Anish Athalye, Nicholas Carlini, and David~A. Wagner.
\newblock Obfuscated gradients give a false sense of security: Circumventing
  defenses to adversarial examples.
\newblock In Jennifer~G. Dy and Andreas Krause, editors, \emph{Proceedings of
  the 35th International Conference on Machine Learning, {ICML} 2018,
  Stockholmsm{\"{a}}ssan, Stockholm, Sweden, July 10-15, 2018}, volume~80 of
  \emph{Proceedings of Machine Learning Research}, pages 274--283. {PMLR},
  2018.
\newblock URL \url{http://proceedings.mlr.press/v80/athalye18a.html}.

\bibitem[Ben{-}David et~al.(2009)Ben{-}David, P{\'{a}}l, and
  Shalev{-}Shwartz]{DBLP:conf/colt/Ben-DavidPS09}
Shai Ben{-}David, D{\'{a}}vid P{\'{a}}l, and Shai Shalev{-}Shwartz.
\newblock Agnostic online learning.
\newblock In \emph{{COLT} 2009 - The 22nd Conference on Learning Theory,
  Montreal, Quebec, Canada, June 18-21, 2009}, 2009.
\newblock URL
  \url{http://www.cs.mcgill.ca/\%7Ecolt2009/papers/032.pdf\#page=1}.

\bibitem[Blum and Monsour(2007)]{blum2007learning}
Avrim Blum and Yishay Monsour.
\newblock Learning, regret minimization, and equilibria.
\newblock 2007.

\bibitem[Blumer et~al.(1989)Blumer, Ehrenfeucht, Haussler, and
  Warmuth]{blumer:89}
A.~Blumer, A.~Ehrenfeucht, D.~Haussler, and M.~Warmuth.
\newblock Learnability and the {Vapnik-Chervonenkis} dimension.
\newblock \emph{Journal of the Association for Computing Machinery},
  36\penalty0 (4):\penalty0 929--965, 1989.

\bibitem[Bousquet et~al.(2020)Bousquet, Hanneke, Moran, and
  Zhivotovskiy]{DBLP:conf/colt/BousquetHMZ20}
Olivier Bousquet, Steve Hanneke, Shay Moran, and Nikita Zhivotovskiy.
\newblock Proper learning, helly number, and an optimal {SVM} bound.
\newblock In Jacob~D. Abernethy and Shivani Agarwal, editors, \emph{Conference
  on Learning Theory, {COLT} 2020, 9-12 July 2020, Virtual Event [Graz,
  Austria]}, volume 125 of \emph{Proceedings of Machine Learning Research},
  pages 582--609. {PMLR}, 2020.
\newblock URL \url{http://proceedings.mlr.press/v125/bousquet20a.html}.

\bibitem[Cesa{-}Bianchi et~al.(2004)Cesa{-}Bianchi, Conconi, and
  Gentile]{DBLP:journals/tit/Cesa-BianchiCG04}
Nicol{\`{o}} Cesa{-}Bianchi, Alex Conconi, and Claudio Gentile.
\newblock On the generalization ability of on-line learning algorithms.
\newblock \emph{{IEEE} Trans. Inf. Theory}, 50\penalty0 (9):\penalty0
  2050--2057, 2004.
\newblock \doi{10.1109/TIT.2004.833339}.
\newblock URL \url{https://doi.org/10.1109/TIT.2004.833339}.

\bibitem[Cohen et~al.(2019)Cohen, Rosenfeld, and
  Kolter]{DBLP:conf/icml/CohenRK19}
Jeremy~M. Cohen, Elan Rosenfeld, and J.~Zico Kolter.
\newblock Certified adversarial robustness via randomized smoothing.
\newblock In Kamalika Chaudhuri and Ruslan Salakhutdinov, editors,
  \emph{Proceedings of the 36th International Conference on Machine Learning,
  {ICML} 2019, 9-15 June 2019, Long Beach, California, {USA}}, volume~97 of
  \emph{Proceedings of Machine Learning Research}, pages 1310--1320. {PMLR},
  2019.
\newblock URL \url{http://proceedings.mlr.press/v97/cohen19c.html}.

\bibitem[David et~al.(2016)David, Moran, and Yehudayoff]{david:16}
O.~David, S.~Moran, and A.~Yehudayoff.
\newblock Supervised learning through the lens of compression.
\newblock In \emph{Advances in Neural Information Processing Systems 29}, pages
  2784--2792, 2016.

\bibitem[Engstrom et~al.(2019)Engstrom, Tran, Tsipras, Schmidt, and
  Madry]{pmlr-v97-engstrom19a}
Logan Engstrom, Brandon Tran, Dimitris Tsipras, Ludwig Schmidt, and Aleksander
  Madry.
\newblock Exploring the landscape of spatial robustness.
\newblock In Kamalika Chaudhuri and Ruslan Salakhutdinov, editors,
  \emph{Proceedings of the 36th International Conference on Machine Learning},
  volume~97 of \emph{Proceedings of Machine Learning Research}, pages
  1802--1811, Long Beach, California, USA, 09--15 Jun 2019. PMLR.
\newblock URL \url{http://proceedings.mlr.press/v97/engstrom19a.html}.

\bibitem[GALLANT(1986)]{10008965845}
S.~I. GALLANT.
\newblock Optimal linear discriminants.
\newblock \emph{Eighth International Conference on Pattern Recognition}, pages
  849--852, 1986.
\newblock URL \url{https://ci.nii.ac.jp/naid/10008965845/en/}.

\bibitem[Goldwasser et~al.(2020)Goldwasser, Kalai, Kalai, and
  Montasser]{DBLP:conf/nips/GoldwasserKKM20}
Shafi Goldwasser, Adam~Tauman Kalai, Yael Kalai, and Omar Montasser.
\newblock Beyond perturbations: Learning guarantees with arbitrary adversarial
  test examples.
\newblock In Hugo Larochelle, Marc'Aurelio Ranzato, Raia Hadsell,
  Maria{-}Florina Balcan, and Hsuan{-}Tien Lin, editors, \emph{Advances in
  Neural Information Processing Systems 33: Annual Conference on Neural
  Information Processing Systems 2020, NeurIPS 2020, December 6-12, 2020,
  virtual}, 2020.
\newblock URL
  \url{https://proceedings.neurips.cc/paper/2020/hash/b6c8cf4c587f2ead0c08955ee6e2502b-Abstract.html}.

\bibitem[Goodfellow et~al.(2015)Goodfellow, Shlens, and
  Szegedy]{DBLP:journals/corr/GoodfellowSS14}
Ian~J. Goodfellow, Jonathon Shlens, and Christian Szegedy.
\newblock Explaining and harnessing adversarial examples.
\newblock In Yoshua Bengio and Yann LeCun, editors, \emph{3rd International
  Conference on Learning Representations, {ICLR} 2015, San Diego, CA, USA, May
  7-9, 2015, Conference Track Proceedings}, 2015.
\newblock URL \url{http://arxiv.org/abs/1412.6572}.

\bibitem[Graepel et~al.(2005)Graepel, Herbrich, and Shawe-{T}aylor]{graepel:05}
T.~Graepel, R.~Herbrich, and J.~Shawe-{T}aylor.
\newblock {PAC}-{B}ayesian compression bounds on the prediction error of
  learning algorithms for classification.
\newblock \emph{Machine Learning}, 59\penalty0 (1-2):\penalty0 55--76, 2005.

\bibitem[Hanneke et~al.(2021)Hanneke, Livni, and Moran]{hanneke2021online}
Steve Hanneke, Roi Livni, and Shay Moran.
\newblock Online learning with simple predictors and a combinatorial
  characterization of minimax in 0/1 games, 2021.

\bibitem[Hodges(1997)]{DBLP:books/daglib/0030198}
Wilfrid Hodges.
\newblock \emph{A Shorter Model Theory}.
\newblock Cambridge University Press, 1997.
\newblock ISBN 978-0-521-58713-6.

\bibitem[Kang et~al.(2019)Kang, Sun, Hendrycks, Brown, and
  Steinhardt]{DBLP:journals/corr/abs-1908-08016}
Daniel Kang, Yi~Sun, Dan Hendrycks, Tom Brown, and Jacob Steinhardt.
\newblock Testing robustness against unforeseen adversaries.
\newblock \emph{CoRR}, abs/1908.08016, 2019.
\newblock URL \url{http://arxiv.org/abs/1908.08016}.

\bibitem[Laidlaw et~al.(2020)Laidlaw, Singla, and
  Feizi]{DBLP:journals/corr/abs-2006-12655}
Cassidy Laidlaw, Sahil Singla, and Soheil Feizi.
\newblock Perceptual adversarial robustness: Defense against unseen threat
  models.
\newblock \emph{CoRR}, abs/2006.12655, 2020.
\newblock URL \url{https://arxiv.org/abs/2006.12655}.

\bibitem[L{\'{e}}cuyer et~al.(2019)L{\'{e}}cuyer, Atlidakis, Geambasu, Hsu, and
  Jana]{DBLP:conf/sp/LecuyerAG0J19}
Mathias L{\'{e}}cuyer, Vaggelis Atlidakis, Roxana Geambasu, Daniel Hsu, and
  Suman Jana.
\newblock Certified robustness to adversarial examples with differential
  privacy.
\newblock In \emph{2019 {IEEE} Symposium on Security and Privacy, {SP} 2019,
  San Francisco, CA, USA, May 19-23, 2019}, pages 656--672. {IEEE}, 2019.
\newblock \doi{10.1109/SP.2019.00044}.
\newblock URL \url{https://doi.org/10.1109/SP.2019.00044}.

\bibitem[Littlestone(1987)]{DBLP:journals/ml/Littlestone87}
Nick Littlestone.
\newblock Learning quickly when irrelevant attributes abound: {A} new
  linear-threshold algorithm.
\newblock \emph{Mach. Learn.}, 2\penalty0 (4):\penalty0 285--318, 1987.
\newblock \doi{10.1007/BF00116827}.
\newblock URL \url{https://doi.org/10.1007/BF00116827}.

\bibitem[Littlestone and Warmuth(1994)]{DBLP:journals/iandc/LittlestoneW94}
Nick Littlestone and Manfred~K. Warmuth.
\newblock The weighted majority algorithm.
\newblock \emph{Inf. Comput.}, 108\penalty0 (2):\penalty0 212--261, 1994.
\newblock \doi{10.1006/inco.1994.1009}.
\newblock URL \url{https://doi.org/10.1006/inco.1994.1009}.

\bibitem[Madry et~al.(2018)Madry, Makelov, Schmidt, Tsipras, and
  Vladu]{DBLP:conf/iclr/MadryMSTV18}
Aleksander Madry, Aleksandar Makelov, Ludwig Schmidt, Dimitris Tsipras, and
  Adrian Vladu.
\newblock Towards deep learning models resistant to adversarial attacks.
\newblock In \emph{6th International Conference on Learning Representations,
  {ICLR} 2018, Vancouver, BC, Canada, April 30 - May 3, 2018, Conference Track
  Proceedings}. OpenReview.net, 2018.
\newblock URL \url{https://openreview.net/forum?id=rJzIBfZAb}.

\bibitem[Maini et~al.(2020)Maini, Wong, and Kolter]{DBLP:conf/icml/MainiWK20}
Pratyush Maini, Eric Wong, and J.~Zico Kolter.
\newblock Adversarial robustness against the union of multiple perturbation
  models.
\newblock In \emph{Proceedings of the 37th International Conference on Machine
  Learning, {ICML} 2020, 13-18 July 2020, Virtual Event}, volume 119 of
  \emph{Proceedings of Machine Learning Research}, pages 6640--6650. {PMLR},
  2020.
\newblock URL \url{http://proceedings.mlr.press/v119/maini20a.html}.

\bibitem[Montasser et~al.(2019)Montasser, Hanneke, and
  Srebro]{pmlr-v99-montasser19a}
Omar Montasser, Steve Hanneke, and Nathan Srebro.
\newblock Vc classes are adversarially robustly learnable, but only improperly.
\newblock In Alina Beygelzimer and Daniel Hsu, editors, \emph{Proceedings of
  the Thirty-Second Conference on Learning Theory}, volume~99 of
  \emph{Proceedings of Machine Learning Research}, pages 2512--2530, Phoenix,
  USA, 25--28 Jun 2019. PMLR.

\bibitem[Montasser et~al.(2020{\natexlab{a}})Montasser, Goel, Diakonikolas, and
  Srebro]{DBLP:conf/icml/MontasserGDS20}
Omar Montasser, Surbhi Goel, Ilias Diakonikolas, and Nathan Srebro.
\newblock Efficiently learning adversarially robust halfspaces with noise.
\newblock In \emph{Proceedings of the 37th International Conference on Machine
  Learning, {ICML} 2020, 13-18 July 2020, Virtual Event}, volume 119 of
  \emph{Proceedings of Machine Learning Research}, pages 7010--7021. {PMLR},
  2020{\natexlab{a}}.
\newblock URL \url{http://proceedings.mlr.press/v119/montasser20a.html}.

\bibitem[Montasser et~al.(2020{\natexlab{b}})Montasser, Hanneke, and
  Srebro]{DBLP:conf/nips/MontasserHS20}
Omar Montasser, Steve Hanneke, and Nati Srebro.
\newblock Reducing adversarially robust learning to non-robust {PAC} learning.
\newblock In Hugo Larochelle, Marc'Aurelio Ranzato, Raia Hadsell,
  Maria{-}Florina Balcan, and Hsuan{-}Tien Lin, editors, \emph{Advances in
  Neural Information Processing Systems 33: Annual Conference on Neural
  Information Processing Systems 2020, NeurIPS 2020, December 6-12, 2020,
  virtual}, 2020{\natexlab{b}}.
\newblock URL
  \url{https://proceedings.neurips.cc/paper/2020/hash/a822554e5403b1d370db84cfbc530503-Abstract.html}.

\bibitem[Moran and Yehudayoff(2016)]{moran:16}
S.~Moran and A.~Yehudayoff.
\newblock Sample compression schemes for {VC} classes.
\newblock \emph{Journal of the {ACM}}, 63\penalty0 (3):\penalty0 21:1--21:10,
  2016.

\bibitem[Raghunathan et~al.(2018{\natexlab{a}})Raghunathan, Steinhardt, and
  Liang]{raghunathan2018certified}
Aditi Raghunathan, Jacob Steinhardt, and Percy Liang.
\newblock Certified defenses against adversarial examples.
\newblock In \emph{6th International Conference on Learning Representations,
  {ICLR} 2018, Vancouver, BC, Canada, April 30 - May 3, 2018, Conference Track
  Proceedings}. OpenReview.net, 2018{\natexlab{a}}.
\newblock URL \url{https://openreview.net/forum?id=Bys4ob-Rb}.

\bibitem[Raghunathan et~al.(2018{\natexlab{b}})Raghunathan, Steinhardt, and
  Liang]{raghunathan2018semidefinite}
Aditi Raghunathan, Jacob Steinhardt, and Percy Liang.
\newblock Semidefinite relaxations for certifying robustness to adversarial
  examples.
\newblock In Samy Bengio, Hanna~M. Wallach, Hugo Larochelle, Kristen Grauman,
  Nicol{\`{o}} Cesa{-}Bianchi, and Roman Garnett, editors, \emph{Advances in
  Neural Information Processing Systems 31: Annual Conference on Neural
  Information Processing Systems 2018, NeurIPS 2018, December 3-8, 2018,
  Montr{\'{e}}al, Canada}, pages 10900--10910, 2018{\natexlab{b}}.
\newblock URL
  \url{https://proceedings.neurips.cc/paper/2018/hash/29c0605a3bab4229e46723f89cf59d83-Abstract.html}.

\bibitem[Salman et~al.(2019)Salman, Li, Razenshteyn, Zhang, Zhang, Bubeck, and
  Yang]{DBLP:conf/nips/SalmanLRZZBY19}
Hadi Salman, Jerry Li, Ilya~P. Razenshteyn, Pengchuan Zhang, Huan Zhang,
  S{\'{e}}bastien Bubeck, and Greg Yang.
\newblock Provably robust deep learning via adversarially trained smoothed
  classifiers.
\newblock In Hanna~M. Wallach, Hugo Larochelle, Alina Beygelzimer, Florence
  d'Alch{\'{e}}{-}Buc, Emily~B. Fox, and Roman Garnett, editors, \emph{Advances
  in Neural Information Processing Systems 32: Annual Conference on Neural
  Information Processing Systems 2019, NeurIPS 2019, December 8-14, 2019,
  Vancouver, BC, Canada}, pages 11289--11300, 2019.
\newblock URL
  \url{https://proceedings.neurips.cc/paper/2019/hash/3a24b25a7b092a252166a1641ae953e7-Abstract.html}.

\bibitem[Sauer(1972)]{sauer:72}
N.~Sauer.
\newblock On the density of families of sets.
\newblock \emph{Journal of Combinatorial Theory (A)}, 13\penalty0 (1):\penalty0
  145--147, 1972.

\bibitem[Schapire and Freund(2012)]{schapire:12}
R.~E. Schapire and Y.~Freund.
\newblock \emph{Boosting}.
\newblock Adaptive Computation and Machine Learning. MIT Press, Cambridge, MA,
  2012.

\bibitem[Schapire(2006)]{robschapire}
Rob Schapire.
\newblock Lecture notes - cos 511: Foundations of machine learning.
\newblock March 2006.

\bibitem[Schmidt et~al.(2018)Schmidt, Santurkar, Tsipras, Talwar, and
  Madry]{schmidt2018adversarially}
Ludwig Schmidt, Shibani Santurkar, Dimitris Tsipras, Kunal Talwar, and
  Aleksander Madry.
\newblock Adversarially robust generalization requires more data.
\newblock In Samy Bengio, Hanna~M. Wallach, Hugo Larochelle, Kristen Grauman,
  Nicol{\`{o}} Cesa{-}Bianchi, and Roman Garnett, editors, \emph{Advances in
  Neural Information Processing Systems 31: Annual Conference on Neural
  Information Processing Systems 2018, NeurIPS 2018, December 3-8, 2018,
  Montr{\'{e}}al, Canada}, pages 5019--5031, 2018.
\newblock URL
  \url{https://proceedings.neurips.cc/paper/2018/hash/f708f064faaf32a43e4d3c784e6af9ea-Abstract.html}.

\bibitem[Shelah(1990)]{DBLP:books/daglib/0067545}
Saharon Shelah.
\newblock \emph{Classification theory - and the number of non-isomorphic
  models, Second Edition}, volume~92 of \emph{Studies in logic and the
  foundations of mathematics}.
\newblock North-Holland, 1990.
\newblock ISBN 978-0-444-70260-9.

\bibitem[Tram{\`{e}}r and Boneh(2019)]{DBLP:conf/nips/TramerB19}
Florian Tram{\`{e}}r and Dan Boneh.
\newblock Adversarial training and robustness for multiple perturbations.
\newblock In Hanna~M. Wallach, Hugo Larochelle, Alina Beygelzimer, Florence
  d'Alch{\'{e}}{-}Buc, Emily~B. Fox, and Roman Garnett, editors, \emph{Advances
  in Neural Information Processing Systems 32: Annual Conference on Neural
  Information Processing Systems 2019, NeurIPS 2019, December 8-14, 2019,
  Vancouver, BC, Canada}, pages 5858--5868, 2019.
\newblock URL
  \url{http://papers.nips.cc/paper/8821-adversarial-training-and-robustness-for-multiple-perturbations}.

\bibitem[Vapnik and Chervonenkis(1971)]{vapnik:71}
V.~Vapnik and A.~Chervonenkis.
\newblock On the uniform convergence of relative frequencies of events to their
  probabilities.
\newblock \emph{Theory of Probability and its Applications}, 16\penalty0
  (2):\penalty0 264--280, 1971.

\bibitem[Vapnik and Chervonenkis(1974)]{vapnik:74}
V.~Vapnik and A.~Chervonenkis.
\newblock \emph{Theory of Pattern Recognition}.
\newblock Nauka, Moscow, 1974.

\bibitem[Wong and Kolter(2018)]{wong2018provable}
Eric Wong and Zico Kolter.
\newblock Provable defenses against adversarial examples via the convex outer
  adversarial polytope.
\newblock In \emph{International Conference on Machine Learning}, pages
  5283--5292, 2018.

\end{thebibliography}

\appendix
\appendix
\section{Lemma and Proof of \prettyref{thm:realizable-littlestone-2}}
\label{app:appendix-a}

\begin{lem} [Robust Generalization with Stable Sample Compression]
\label{lem:stable-robust-compression}
Let $(\kappa,\rho)$ be a stable sample compression scheme of size $k$ for $\calH$ with respect to the robust loss $\sup_{z\in\calU(x)} \ind[ h(z)\neq y]$. Then, for any distribution $\calD$ over $\calX \times \calY$ such that $\inf_{h\in\calH}\Risk_{\calU}(h;\calD)=0$, any integer $m > 2k$, and any $\delta\in (0,1)$, with probability at least $1-\delta$ over $S = \{(x_{1},y_{1}),\ldots,(x_{m},y_{m})\}$ iid $\calD$-distributed random variables,
\[\Risk_{\calU}(\rho(\kappa(S)); \calD) \leq \frac{2}{m-2k}\inparen{k\ln(4)+\ln\inparen{\frac{1}{\delta}}}.\]
\end{lem}

\begin{proof}
The argument follows an analogous proof from \cite{DBLP:conf/colt/BousquetHMZ20} for the $0$-$1$ loss. We observe that the same argument applies to the robust loss, and we provide an explicit proof for completeness. Split the $m$ samples of $S$ into $2k$ sets $S_1,\dots, S_{2k}$ each of size $\frac{m}{2k}$. Observe that the $k$ compression points chosen by $\kappa$, $\kappa(S)$, are in at most $k$ of these sets $S_{i^*_1}, \dots, S_{i^*_k}$ where $i^*_1,\dots,i^*_k\in \sett{1,\dots,2k}$. Stability of $(\kappa,\rho)$ implies that $\rho(\kappa(\cup_{j=1}^{k}S_{i^*_j}))=\rho(\kappa(S))$. Since by definition of $(\kappa,\rho)$, the robust risk $\Risk_{\calU}(\rho(\kappa(S));S)=0$, it follows that $\Risk_{\calU}(\rho(\kappa(\cup_{j=1}^{k}S_{i^*_j})));S)=0$. This implies that $\rho(\kappa(\cup_{j=1}^{k}S_{i^*_j}))$ is robustly correct on the remaining sets $\cup_{j\notin \sett{i^*_1,\dots,i^*_k}} S_j$. 

Observe that the event that $\Risk_{\calU}(\rho(\kappa(S)); \calD) > \eps$ implies the event that there exists $i_1,\dots, i_k \in \sett{1,\dots,2k}$ such that $\Risk_{\calU}(\rho(\kappa(\cup_{j=1}^{k} S_{i_j})); \calD) > \eps$ and $\rho(\kappa(\cup_{j=1}^{k} S_{i_j}))$ robustly correct on $\cup_{j\notin\sett{i_1,\dots,i_k}}S_j$. Thus,
\begin{align*}
    &\Prob_{S\sim \calD^{m}}\insquare{ \Risk_{\calU}(\rho(\kappa(S)); \calD) > \eps }\\
    &\leq \Prob_{S\sim\calD^m}\insquare{\exists i_1,\dots, i_k: \Risk_{\calU}(\rho(\kappa(\cup_{j=1}^{k} S_{i_j})); \calD) > \eps \wedge \Risk_\calU(\rho(\kappa(\cup_{j=1}^{k} S_{i_j}));\cup_{j\notin\sett{i_1,\dots,i_k}}S_j)=0}\\
    &\overset{(i)}{\leq} {2k \choose k} \Prob_{S\sim \calD^m} \insquare{\Risk_{\calU}(\rho(\kappa(\cup_{j=1}^{k} S_{i_j})); \calD) > \eps \wedge \Risk_\calU(\rho(\kappa(\cup_{j=1}^{k} S_{i_j}));\cup_{j\notin\sett{i_1,\dots,i_k}}S_j)=0}\\
    &\overset{(ii)}{\leq} {2k \choose k} \inparen{1-\eps}^{m/2} < 4^{k}e^{-\eps m/2},
\end{align*}
where inequality $(i)$ follows from a union bound, and inequality $(ii)$ follows from observing that the $\frac{m}{2}$ samples in $\cup_{j\notin\sett{i_1,\dots,i_k}}S_j$ are independent of $\rho(\kappa(\cup_{j=1}^{k} S_{i_j})$. Setting $4^{k}e^{-\eps m/2}=\delta$ and solving for $\eps$ yields the stated bound.
\end{proof}

\begin{proof}[of \prettyref{thm:realizable-littlestone-2}]
Let $\calA: (\calX\times \calY)^* \to \calY^\calX$ be a conservative online learner for $\calH$ with mistake bound equal to $\Ldim(\calH)$. Let $\calU:\calX\to 2^\calX$ be an arbitrary adversarial set that is unknown to the learning algorithm and $\sfO_\calU$ a black-box perfect attack oracle for $\calU$. Let $\calD$ be an arbitrary distribution over $\calX\times \calY$ that is robustly realizable with some concept $h^*\in \calH$, i.e., $\Risk_\calU(h^*;\calD)=0$. Fix $\eps, \delta \in (0,1)$ and a sample size $m$ that will be determined later. Let $S=\sett{(x_1,y_1),\dots,(x_m,y_m)}$ be an iid sample from $\calD$. Our proof will be divided into two main parts.

\paragraph{Zero Empirical Robust Loss} Observe that the output of \alg{\cycalg} (\prettyref{alg:robust_learner_unknown}): $\hat{h} = \calB(S, \sfO_\calU, \calA)$, achieves zero robust loss on the training data, $\Risk_\calU(\hat{h};S)=0$. This follows because whenever \alg{\cycalg} (\prettyref{alg:robust_learner_unknown}) terminates, Steps 4-11 imply that it made a full pass on dataset $S$ without encountering any example $(x_i,y_i)$ where predictor $\hat{h}$ is not robustly correct. Furthermore, since conservative online learner $\calA$ has a finite mistake bound of $\Ldim(\calH)$, it implies that the number of full passes (execution of Step 3) \prettyref{alg:robust_learner_unknown} makes over $S$ is at most $\Ldim(\calH)$, and in each pass $m$ oracle queries to $\sfO_\calU$ are made. Thus, with at most $m\Ldim(\calH)$ oracle queries to $\sfO_\calU$, \alg{\cycalg} (\prettyref{alg:robust_learner_unknown}) outputs a predictor $\hat{h}$ with zero robust loss on $S$, $\Risk_\calU(\hat{h};S)=0$. 

\paragraph{Robust Generalization through Stable Sample Compression} \alg{\cycalg} (\prettyref{alg:robust_learner_unknown}) can be viewed as a stable compression scheme for the robust loss. Specifically, the output of the compression function $\kappa(S,\sfO_\calU, \calA)$ is an order-dependent sequence that contains all examples $(x_i,y_i)$ on which $\hat{h}$ was not robustly correct while cycling through dataset $S$ (Steps 6-7), since $\calA$ has a finite mistake bound of $\Ldim(\calH)$, it follows that $\abs{\kappa(S,\sfO_\calU, \calA)}\leq \Ldim(\calH)$. The reconstruction function $\rho$ simply runs \alg{\cycalg} (\prettyref{alg:robust_learner_unknown}) on the compressed dataset $S'= \kappa(S,\sfO_\calU, \calA)$. The fact that $\calA$ is a conservative online learner implies that $\hat{h} = \calB(S,\sfO_\calU,\calA) = \calB(S',\sfO_\calU,\calA)$. Since $\Risk_\calU(\hat{h};S)=0$, this establishes that $(\kappa,\rho)$ is a sample compression scheme for the robust loss. Furthermore, since $\calA$ is a conservative online learner, observe that for any $S''$ such that $\kappa(S,\sfO_\calU,\calA)\subseteq S'' \subseteq S$ it holds that $\kappa(S,\sfO_\calU,\calA) = \kappa(S'',\sfO_\calU,\calA)$. That is, removing any of the examples from $S$ on which $\hat{h}$ was robustly correct in Step 6 will not change the output of the compression function $\kappa$. Thus, the pair $(\kappa,\rho)$ is a stable sample compression scheme for the robust loss of size $\Ldim(\calH)$. To conclude the proof, \prettyref{lem:stable-robust-compression} guarantees that for a sample size $m(\eps,\delta)=O\inparen{\frac{\Ldim(\calH)+\log(1/\delta)}{\eps}}$, the robust risk $\Risk_\calU(\hat{h};\calD)\leq \eps$. 
\end{proof}

\section{Auxiliary Lemmas and Proof of \prettyref{thm:realizable-strongust}}
\label{app:appendix-vc-upperbound}

\begin{lem}[Properties of $\alpha$-Boost, see, e.g., Corollary 6.4 and Section 6.4.3 in \cite{schapire:12}]
\label{lem:alphaboost}
Let $S=\sett{(x_i, c(x_i))}_{i=1}^{m}$ be a dataset where $c \in \calC$ is some target concept, and $\calA$ an arbitrary PAC learner for $\calC$ (for $\eps = 1/3$, $\delta = 1/3$). Then, running $\alpha$-Boost on $S$ with black-box oracle access to $\calA$ with $\alpha = \frac{1}{2}\ln\inparen{1+\sqrt{\frac{2\ln m}{T}}}$ for $T=\lceil 112 \ln(m) \rceil = O(\log m )$ rounds suffices to produce a sequence of hypotheses $h_{1},\ldots,h_{T} \in {\rm im}(\calA)$ such that 
\[\forall (x,y) \in S, \frac{1}{T} \sum_{i=1}^{T} \ind[ h_{i}(x) = y ] \geq \frac{5}{9}.\]
In particular, this implies that the majority-vote $\MAJ(h_1,\dots,h_T)$ achieves zero error on $S$. 
\end{lem}

\begin{lem} [Sparsification of Majority Votes, \cite{moran:16}]
\label{lem:sparse}
Let $\calH$ be a hypothesis class with finite primal and dual VC dimension, and $h_1,\dots, h_T$ be predictors in $\calH$. Then, for any $(\eps,\delta) \in (0,1)$, with probability at least $1-\delta$ over $N=O\inparen{\frac{\vc^*(\calH) + \log(1/\delta)}{\eps^2}}$ independent random indices $i_{1},\ldots,i_{N} \sim {\rm Uniform}(\{1,\ldots,T\})$,
we have:
\[
\forall (x,y)\in \calX\times \calY, \abs{ \frac{1}{N} \sum\limits_{j=1}^{N} \ind[ h_{i_{j}}(x) = y ] -  \frac{1}{T} \sum\limits_{i=1}^{T} \ind[ h_{i}(x) = y ]} < \eps.
\]
\end{lem}

\begin{lem} [\cite{pmlr-v99-montasser19a}]
\label{lem:robust-compression}
For any $k \in \bbN$ and fixed function $\phi : (\calX \times \calY)^{k} \to \calY^{\calX}$, for any distribution $P$ over $\calX \times \calY$ and any $m \in \bbN$, 
for $S = \{(x_{1},y_{1}),\ldots,(x_{m},y_{m})\}$ iid $P$-distributed random variables,
with probability at least $1-\delta$, 
if $\exists i_{1},\ldots,i_{k} \in \{1,\ldots,m\}$ 
s.t.\ $\hat{R}_{\calU}(\phi((x_{i_{1}},y_{i_{1}}),\ldots,(x_{i_{k}},y_{i_{k}}));S) = 0$, 
then 
\begin{equation*}
\Risk_{\calU}(\phi((x_{i_{1}},y_{i_{1}}),\ldots,(x_{i_{k}},y_{i_{k}}));P) \leq \frac{1}{m-k} (k\ln(m) + \ln(1/\delta)).
\end{equation*}
\end{lem}

\begin{lem} [\cite{DBLP:conf/nips/MontasserHS20}]
\label{lem:dualvc-convexhull}
Let ${\rm co}^k(\calH)=\sett{ x \mapsto \MAJ(h_1,\dots,h_k)(x): h_1,\dots,h_k \in \calH}$. Then, the dual VC dimension of $\co^k(\calH)$ satisfies $\vc^*(\co^k(\calH)) \leq O(d^* \log k)$.
\end{lem}

\begin{thm} [Weak Robust Learner]
\label{thm:realizable-weak-robust}
For any class $\calH$ with $\vc(\calH)=d$ and $\vc^*(\calH)=d$, \alg{\rlua}(\prettyref{alg:robust_learner_unknown}) robustly PAC learns $\calH$ w.r.t any $\calU$ with:
\begin{enumerate}
    \item Sample Complexity $m(\eps,\delta)=O\inparen{ \frac{d{d^*}^2\log^2d^*}{\eps}\log\inparen{\frac{d{d^*}^2\log^2d^*}{\eps}} + \frac{ \log(1/\delta)}{\eps}}.$
    \item Oracle Complexity $T(\eps,\delta)=O\inparen{m(\eps,\delta)^{d{d^*}^2\log^2d^*} + m(\eps,\delta)^{d{d^*}\log d^*}\Ldim(\calH)}.$
\end{enumerate}
\end{thm}

\begin{proof}
Let $\calA:(\calX\times \calY)^*\to \calY^\calX$ be an online learner for $\calH$ with mistake bound $M(\calA,\calH)=O(\Ldim(\calH))$. We do not require $\calA$ to be ``proper'' (i.e.~returns a predictor in $\calH$), but we will rely on it returning a predictor in some, possibly much larger, class which still has finite VC-dimension.  To this end, we denote by $\vc(\calA)=\vc(\im(\calA))$ and $\vc^*(\calA)=\vc^*(\im(\calA))$ the primal and dual VC dimension of the image of $\calA$, i.e.~the class $\im(\calA) = \left\{ \calA(S) \middle| S \in (\calX\times\calY)^* \right\}$ of the possible hypothesis $\calA$ might return. We will first prove a sample and oracle complexity bound stated in terms of  $\vc(\calA)$ and $\vc^*(\calA)$, and later, at the end of the proof, we will use a result due to \citep{hanneke2021online} to bound $\vc(\calA)$ and $\vc^*(\calA)$ in terms of $d=\vc(\calH)$ and $d^*=\vc(\calH)$ for a specific online learner $\calA$.

Let $\calU:\calX\to 2^\calX$ be an arbitrary adversary that is unknown to the learner. Let $\calD$ be an arbitrary distribution over $\calX\times \calY$ that is robustly realizable with some concept $h^*\in \calH$, i.e., $\Risk_\calU(h^*;\calD)=0$. Fix $\eps, \delta \in (0,1)$ and a sample size $m$ that will be determined later. Let $S=\sett{(x_1,y_1),\dots,(x_m,y_m)}$ be an iid sample from $\calD$.

\paragraph{Zero Empirical Robust Loss.} Let $L\subseteq S$. Let $\calA_{\rm cyc}$ be \alg{\cycalg} (\prettyref{alg:robust_learner_unknown}) from \prettyref{thm:realizable-littlestone-2}. By \prettyref{thm:realizable-littlestone-2}, running $\calA_{\rm cyc}$ on input $L$ with black-box access to $\sfO_\calU$ and black-box access to $\calA$, guarantees that the output $\hat{h}=\calA_{\rm cyc}(L,\sfO_\calU, \calA)$ satisfies $\Risk_\calU(\hat{h};L)=0$ with at most $\abs{L}\Ldim(\calH)$ oracle queries to $\sfO_\calU$.

\paragraph{Discretization} Before we can apply the compression approach, we will inflate dataset $S$ to a (potentially infinite) larger dataset $S_\calU = \bigcup_{i\leq m} \sett{ (z, y_i): z\in \calU(x_i)}$ that includes all possible adversarial perturbations under $\calU$. There are two challenges that need to be addressed. First, $S_\calU$ can be potentially infinite, and so we would need to discretize it somehow. Second, the learner does not know $\calU$ and so the inflation can be carried only through interaction with the perfect attack oracle $\sfO_\calU$. Denote by $\hat{\calH}=\sett{\calA_{\rm cyc}(L): L\subseteq S, \abs{L}=n}$ where $n=O(\vc(\calA))$. Think of $\hat{\calH}$ as the effective hypothesis class that is used by our robust learning algorithm $\mathcal{B}$ that we are constructing. Note that $|\hat{\calH}|\leq \abs{\sett{L:L\subseteq S, \abs{L}=n}}= {m \choose n}\leq \inparen{\frac{em}{n}}^n$. We will now apply classic tools from VC theory to argue that there is a finite number of behaviors when projecting the infinite unknown set $S_\calU$ onto $\hat{\calH}$. Specifically, consider a \emph{dual class} $\calG$: a set of functions $g_{(x,y)} : \hat{\calH} \to \{0,1\}$ defined as $g_{(x,y)}(h) = \ind[ h(x) \neq y ]$, 
for each $h \in \hat{\calH}$ and each $(x,y) \in S_{\calU}$.  
The VC dimension of $\calG$ is at most the \emph{dual VC dimension} of $\hat{\calH}$: $\vc^{*}(\hat{\calH})$, 
which is at most $\vc^*(\calA)$ since $\hat{\calH}\subseteq \im(\calA)$. The set of behaviors when projecting $S_\calU$ onto $\hat{\calH}$ is defined as follows:
\[ S_\calU|_{\hat{\calH}} = \sett{ \inparen{g_{(z,y)}(h_1), \dots, g_{(z,y)}(h_{\abs{\hat{\calH}}})}: (z,y)\in S_\calU }. \]

Now denote by $\hat{S}_{\calU}$ a subset of $S_{\calU}$ which includes exactly one $(z,y) \in S_{\calU}$ 
for each distinct classification $\inparen{g_{(z,y)}(h)}_{h \in \hat{\calH}}$ of $\hat{\calH}$ realized by some $(z,y)\in S_\calU$.
In particular, by applying Sauer's lemma \cite{vapnik:71,sauer:72} on the dual class $\calG$, 
$|\hat{S}_{\calU}|=\abs{S_\calU|_{\hat{\calH}}} \leq \left(\frac{e |\hat{\calH}|}{d^*}\right)^{d^*}$, which is at most $m^{nd^*}$. In particular, note that for any $T \in \bbN$ and $h_{1},\ldots,h_{T} \in \hat{\calH}$, 
if $\frac{1}{T} \sum_{t=1}^{T} \ind[ h_{t}(x) = y ] > \frac{1}{2}$ for every $(z,y) \in \hat{S}_{\calU}$, 
then $\frac{1}{T} \sum_{t=1}^{T} \ind[ h_{t}(x) = y ] > \frac{1}{2}$ for every $(z,y) \in S_{\calU}$ as well, 
which would further imply $\Risk_{\calU}( \MAJ(h_{1},\ldots,h_{T}); S) = 0$. Thus, we have shown that there \textit{exists} a finite discretization $\hat{S}_\calU$ of $S_\calU$ where it suffices to find predictors $h_1,\dots, h_T\in \hat{\calH}$ that achieve zero loss on $\hat{S}_\calU$. 

It remains to show how to construct the discretization $\hat{S}_\calU$ using only interactions with the perfect attack oracle $\sfO_\calU$. To this end, for each $(x,y)\in S$, initialize $P=\sett{(x,y)}$. The robust learner $\calB$ constructs a query $(f_P, (x,y))$ where $f_P:\calX \to \calY$ is a predictor of the form:
\[ f_P(x') = y~\text{if and only if}~\inparen{\exists_{(z,y)\in P}}\inparen{\forall_{h \in \hat{\calH}}} g_{(z,y)}(h)=g_{(x',y)}(h).\]
By the definition of $f_P$, if there is a perturbation $z'\in \calU(x)$ such that the classification pattern $\inparen{g_{(z',y)}(h)}_{h\in \hat{\calH}}$ is distinct from the classification pattern $\inparen{g_{(z,y)}(h)}_{h\in \hat{\calH}}$ of any of the points $(z,y)\in P$, then $f_P(z')\neq y$, and therefore the oracle $\sfO_\calU$ would reveal to the learner perturbation $z'$. Next, the learner adds the point $(z',y)$ to $P$, and repeats the procedure again until $f_P$ is robustly correct on example $(x,y)$. In each oracle query, the learner is forcing the oracle $\sfO_\calU$ to reveal perturbations $z\in \calU(x)$ with distinct classification patterns that the learner did not see before. Since we know that $\abs{\hat{S}_\calU}\leq m^{n\vc^*(\calA)}$, the learner makes at most $m^{n\vc^*(\calA)}$ oracle calls to $\sfO_\calU$ before $f_P$ is robustly correct on $(x,y)$. This process is repeated for each training example $(x,y)\in S$, and so the total number of oracle calls to $\sfO_\calU$ is at most $m^{n\vc^*(\calA)+1}$. 

\paragraph{Oracle Complexity} Our robust learner $\calB$ makes $\inparen{\frac{em}{n}}^nn\Ldim(\calH)$ oracle calls to $\sfO_\calU$ to construct $\hat{\calH}$ and $m^{n\vc^*(\calA)+1}$ oracle calls to $\sfO_\calU$ to construct $\hat{S}_\calU$.

\paragraph{Sample Complexity and Robust Generalization} We proceed by running the sample compression scheme from \cite{pmlr-v99-montasser19a} on the discretized dataset $\hat{S}_\calU$. In this stage no more oracle queries to $\sfO_\calU$ are needed since the learner has already precomputed $\hat{\calH}$ and the discretized dataset $\hat{S}_\calU$. As mentioned above, our goal in this stage is to find predictors $h_1,\dots, h_T \in \hat{\calH}$ where the majority-vote $\MAJ(h_1,\dots, h_T)$ achieves zero loss on $\hat{S}_\calU$. This implies that $\MAJ(h_1,\dots, h_T)$ achieves zero robust loss on $S$, $\Risk_{\calU}( \MAJ(h_{1},\ldots,h_{T}); S) = 0$, by properties of $\hat{\calH}$ and $\hat{S}_\calU$. We will next go about finding such a set of $h_t$ predictors. 

We run the $\alpha$-Boost algorithm on the discretized dataset $\hat{S}_\calU$, this time with $\calA_{\rm cyc}$ (\alg{\cycalg} (\prettyref{alg:robust_learner_unknown})) as the subprocedure. Specifically, on each round of boosting, $\alpha$-Boost computes an empirical distribution $D_t$ over $\hat{S}_\calU$. We draw $n = O(\vc(\calA))$ samples $S'$ from $D_t$, and \textit{project} $S'$ to a dataset $L_t\subset S$ by replacing each perturbation $(z,y)\in S'$ with its corresponding original point $(x,y) \in S$, and then we run $\calA_{\rm cyc}$ on dataset $L_t$ (this is already precomputed since $\calA_{\rm cyc}(L_t)\in \hat{\calH}$ by definition of $\hat{\calH}$). The projection step is crucial for the proof to work, since we use a \emph{sample compression} argument to argue about \textit{robust} generalization, and the sample compression must be done on the \textit{original} points that appeared in $S$ rather than the perturbations in $\hat{S}_\calU$. 

By classic PAC learning guarantees \cite{vapnik:74,blumer:89}, with $n =O(\vc(\calA))$, we are guaranteed uniform convergence of $0$-$1$ risk over predictors in $\hat{\calH}$. So, for any distribution $D$ over $\calX \times \calY$ with $\inf_{h \in \calH} \err(h;\calD)=0$, with nonzero probability over $S'\sim \calD^{n}$, every $h'\in \hat{\calH}$ satisfying $\err_{S'}(h')=0$, also has $\err_D(h')<1/3$. As discussed above, we know that $h_t=\calA_{\rm cyc}(L_t)$ achieves zero robust loss on $L_t$, $\Risk_\calU(h_t; L_t)=0$, which by definition of the projection means that $\err_{S'}(h_t)=0$, and thus $\err_{D_t}(h_t)<1/3$. This allows us to use $\calA_{\rm cyc}$ with $\alpha$-Boost to establish a \textit{robust} generalization guarantee. Specifically, Lemma~\ref{lem:alphaboost} implies that running the $\alpha$-Boost algorithm  
with $\hat{S}_{\calU}$ as its dataset for $T = O(\log(|\hat{S}_{\calU}|))$ rounds, 
using $\calA_{\rm cyc}$ to produce the hypotheses $h_t \in \hat{\calH}$ for the distributions $D_{t}$ produced on each round of the algorithm, will produce a sequence of hypotheses $h_1,\dots,h_T \in \hat{\calH}$ such that:

\[\forall (z,y) \in \hat{S}_{\calU}, \frac{1}{T} \sum_{i=1}^{T} \ind[ h_{i}(z) = y ] \geq \frac{5}{9}.\]

Specifically, this implies that the majority-vote over hypotheses $h_{1},\ldots,h_{T}$ achieves zero {\em robust} loss on dataset $S$, $\Risk_{\calU}(\MAJ(h_1,\dots,h_T); S)=0$. Note that each of these classifiers $h_{t}$ is equal to $\calA(L_{t}, \sfO_\calU)$ for some $L_{t} \subseteq S$ with $|L_{t}|=n$. Thus, the classifier $\MAJ(h_1,\dots,h_T)$ is representable as the value of an (order-dependent) reconstruction function $\phi$ with 
a compression set size 
\[nT = O(n\log(|S_{\calU}|)).\]

We can further reduce the compression set size by sparsifying  the majority-vote. Lemma~\ref{lem:sparse} (with $\eps=1/18,\delta=1/3$) guarantees that for $N=O(\vc^*(\calA))$, the sampled predictors $h_{i_1},\dots, h_{i_{N}}\in \hat{\calH}$ satisfy:
\[\forall (z,y)\in \hat{S}_\calU, \frac{1}{N} \sum\limits_{j=1}^{N} \ind[ h_{i_{j}}(z) = y ] > \frac{1}{T} \sum\limits_{i=1}^{T} \ind[ h_{i}(z) = y ]  - \frac{1}{18} > \frac{5}{9} - \frac{1}{18} = \frac{1}{2},\]
so that the majority-vote achieves zero robust loss on $S$, $\Risk_\calU(\MAJ(h_{i_1},\dots, h_{i_{N}}); S)=0$. Since again, each $h_{i_j}$ is the result of $\calA(L_{t}, \sfO_\calU)$ for some $L_{t} \subseteq S$ with $|L_{t}|=m_0$, we have that the classifier $\MAJ(h_{i_1},\dots, h_{i_{N}})$ can be represented as the value of an (order-dependent) reconstruction function $\phi$ with a compression set size $n N=O(\vc(\calA)\vc^*(\calA))$. Lemma~\ref{lem:robust-compression} (\cite{pmlr-v99-montasser19a}) which extends to the robust loss the classic compression-based generalization guarantees from the $0$-$1$ loss, implies that for $m \geq c\vc(\calA)\vc(\calA)^*$ (for an appropriately large numerical constant $c$), with probability at least $1-\delta$ over $S\sim \calD^m$, 
\begin{equation}
\label{eqn:proof-bound}
    R_{\calU}(\MAJ(h_{i_1}, \dots, h_{i_{N}});\calD) \leq O\!\left( \frac{\vc(\calA)\vc^*(\calA)}{m} \log(m) + \frac{1}{m} \log(1/\delta) \right).
\end{equation}

\paragraph{Bounding the complexity of $\calA$} A result due to \citep[][Theorem~3]{hanneke2021online} states that for any class $\calH$ of Littlestone dimension $\Ldim(\calH)$ and dual VC dimension $d^*$, there is an online learner $\calA$ with mistake bound $M(\calA,\calH)=O(\Ldim(\calH))$ which represents its hypotheses as (unweighted) majority votes of $O(d^*)$ predictors of $\calH$. In other words, 
$$\im(\calA)\subseteq {\rm co}^{O(d^*)}(\calH)\triangleq\sett{ x \mapsto \MAJ(h_1,\dots,h_{O(d^*)})(x): h_1,\dots,h_{O(d^*)} \in \calH}.$$

By \citep{blumer:89}, the VC dimension of ${\rm co}^{O(d^*)}(\calH)$ is at most $O(dd^*\log d^*)$, and by Lemma~\ref{lem:dualvc-convexhull}, the dual VC dimension of ${\rm co}^{O(d^*)}(\calH)$ is at most $O(d^*\log d^*)$. Since $\im(\calA)\subseteq {\rm co}^{O(d^*)}(\calH)$, this implies that $\vc(\calA) = O(dd^*\log d^*)$ and $\vc^*(\calA) = O(d^*\log d^*)$. Substituting these upper bounds in \prettyref{eqn:proof-bound}, and setting it less than $\eps$ and solving for a sufficient size of $m$ yields the stated sample complexity bound.
\end{proof}

\begin{proof} [of \prettyref{thm:realizable-strongust}]
Let $\calB$ be the robust learning algorithm (\prettyref{alg:robust_learner_unknown}) described in \prettyref{thm:realizable-weak-robust}. We will use $\calB$ as a {\em weak} robust learner with fixed parameters $\eps_0 = 1/3$ and $\delta_0=1/3$. By the guarantee of \prettyref{thm:realizable-weak-robust}, with fixed sample complexity $m_0=O(d{d^*}^2\log^2d^*)$, for any distribution $\calD$ over $\calX\times\calY$ such that $\inf_{h\in \calH} \Risk_{\calU}(h;\calD)=0$, with probability at least $1/3$ over $S\sim \calD^{m_0}$, $\Risk_{\calU}(\calB(S);\calD)\leq 1/3$. Furthermore, $\calB$ makes at most $O(d{d^*}^2\log^2d^*)^{O(d{d^*}^2\log^2d^*)}+O(d{d^*}^2\log^2d^*)^{O(d{d^*}\log d^*)}\Ldim(\calH)=\exp\sett{O(d{d^*}^2\log^2d^*)}+\exp\sett{O(d^2{d^*}^2\log^2d^*)}\Ldim(\calH)$ oracle calls to $\sfO_\calU$.

We will now boost the confidence and robust error guarantee of the {\em weak} robust learner $\calB$ by running boosting with respect to the {\em robust} loss (rather than the standard $0$-$1$ loss). Specifically, fix $(\eps,\delta)\in (0,1)$ and a sample size $m(\eps,\delta)$ that will be determined later. Let $S=\sett{(x_1,y_1),\dots,(x_m,y_m)}$ be an iid sample from $\calD$. Run the $\alpha$-Boost algorithm on dataset $S$ using $\calB$ as the weak robust learner for a number of rounds $L=O(\log m)$. On each round $t$, $\alpha$-Boost computes an empirical distribution $D_t$ over $S$ by applying the following update for each $(x,y)\in S$:
\[D_{t}(\sett{(x,y)}) = \frac{D_{t-1}(\sett{(x,y)})}{Z_{t-1}} \times \begin{cases} 
                                                              e^{-2\alpha}& \text{if }\sup_{z\in\calU(x)}\ind [h_{t-1}(z)\neq y]=0\\
                                                              1 &\text{otherwise}
                                                           \end{cases}
\]
where $Z_{t-1}$ is a normalization factor, $\alpha$ is set as in \prettyref{lem:alphaboost}, and $h_{t-1}$ is the {\em weak} robust predictor outputted by $\calB$ on round $t-1$ that satisfies $\Risk_\calU(h_{t-1};D_{t-1})\leq 1/3$. Note that computing $D_t$ requires $\abs{S}=m$ oracle calls to $\sfO_\calU$. Once $D_t$ is computed, we sample $m_0$ examples from $D_t$ and run {\em weak} robust learner $\calB$ on these examples to produce a hypothesis $h_t$ with robust error guarantee $\Risk_\calU(h_{t};D_{t})\leq 1/3$. This step has failure probability at most $\delta_0=1/3$. We will repeat it for at most $\ceil{\log(2L/\delta)}$ times, until $\calB$ succeeds in finding $h_t$ with robust error guarantee $\Risk_\calU(h_{t};D_{t})\leq 1/3$. By a union bound argument, we are guaranteed that with probability at least $1-\delta/2$, foreach $1\leq t\leq L$, $\Risk_\calU(h_{t};D_{t})\leq 1/3$. Furthermore, by \prettyref{lem:alphaboost}, we are guaranteed that $\Risk_{\calU}(\MAJ(h_1,\dots,h_L);S)=0$. Note that each of these classifiers $h_{t}$ is equal to $\calB(S'_{t}, \sfO_\calU)$ for some $S'_{t} \subseteq S$ with $|S'_{t}|=m_0$. Thus, the classifier $\MAJ(h_1,\dots,h_L)$ is representable as the value of an (order-dependent) reconstruction function $\phi$ with 
a compression set size $m_0L=m_0O(\log m)$. Now, invoking \prettyref{lem:robust-compression}, with probability at least $1-\delta/2$,
\[\Risk_{\calU}(\MAJ(h_1,\dots,h_L);\calD)\leq O\inparen{\frac{m_0\log^2 m}{m} + \frac{\log(2/\delta)}{m}},\]
and setting this less than $\eps$ and solving for a sufficient size of $m$ yields the stated sample complexity bound. 

\paragraph{Oracle Complexity} Observe that we run boosting for $L$ rounds, in each round the {\em weak} robust learner is invoked at most $\ceil{\log(2L/\delta)}$ times. In each of these invocations, $\calB$ makes at most $\exp\sett{O(d{d^*}^2\log^2d^*)}+\exp\sett{O(d^2{d^*}^2\log^2d^*)}\Ldim(\calH)$ oracle calls to $\sfO_\calU$, and an additional $m(\eps,\delta)$ oracle calls to $\sfO_\calU$ are made by $\alpha$-Boost to compute the robust error of the $h_t$ hypotheses produced by $\calB$. Thus, the total number of calls to $\sfO_\calU$ is at most
\[\ceil{L\log(2L/\delta)}\inparen{\exp\sett{O(d{d^*}^2\log^2d^*)}+\exp\sett{O(d^2{d^*}^2\log^2d^*)}\Ldim(\calH)+m(\eps,\delta)}.\]
\end{proof}

\section{Proofs for Agnostic Setting -- \prettyref{sec:agnostic-mistakeoracle}}
\label{app:appendix-b}

\begin{algorithm2e}[h]
\caption{Weighted Majority}
\label{alg:weighted-majority}
\SetKwInput{KwInput}{Input}                
\SetKwInput{KwOutput}{Output}              
\SetKwFunction{FExpert}{Expert}
\DontPrintSemicolon
  \KwInput{paramter $\eta \in [0,1)$, black-box perfect attack oracle $\sfO_\calU$, and finite hypothesis class $\calH$.}
Initialize $P_0$ to be uniform over $\calH$, i.e.~$\forall h \in \calH, P_0(h) = 1$.\;
\For{$1 \leq t\leq T$}{
    Receive $(x_t, y_t)$.\;
    Certify the robustness of the weighted-majority-vote $\MAJ_{P_{t-1}}$ on $(x_t,y_t)$ by sending the query $(\MAJ_{P_{t-1}}, (x_t,y_t))$ to the perfect attack oracle $\sfO_\calU$.\;
    \If{$\MAJ_{P_{t-1}}$ is not robustly correct on $(x_t,y_t)$}{
        Let $z_t$ be the perturbation returned by $\sfO_\calU$ where $\MAJ_{P_{t-1}}(z_t)\neq y_t$.\;
        Foreach $h\in\calH$ such that $h(z_t)\neq y_t$, update $P_{t}(h) =  \eta P_{t-1}(h)$. 
    }
}
\KwOutput{The weighted-majority-vote $\MAJ_{P_T}$ over $\calH$.}
\BlankLine
  \SetKwProg{Fn}{}{:}{\KwRet}
  \Fn{\FExpert{Indices $i_1 < i_2 < \cdots < i_L$, and hypothesis class $\calH$}}{
    Initialize $V_1=\calH$.\;
\For{$1\leq t\leq T$}{
    Receive $x_t$.\; 
    Let $V^{y}_t = \sett{ h\in V_t: h(x_t) = y }$ for $y\in\sett{\pm1}$.\;
    Let $\Tilde{y}_t=\argmax_{y\in\sett{\pm1}} \Ldim(V^{y}_{t})$ (in case of a tie set $\Tilde{y}_t=+1$).\;
    \If{$t\in \sett{i_1,\dots, i_L}$}{
        Predict $\hat{y}_t = -\Tilde{y}_t$.\;
    }\Else{
        Predict $\hat{y}_t = \Tilde{y}_t$.\;
    }
    Update $V_{t+1} = V^{\hat{y}_t}_{t}$.\;
    
  }
  }
\end{algorithm2e}

\begin{lem}
\label{lem:agnostic-finite}
For any class $\calH$ with finite cardinality, \alg{Weighted Majority} (\prettyref{alg:weighted-majority}) guarantees that for any $\calU$ and any sequence of examples $(x_1,y_1),\dots, (x_T,y_T)$:
\[\sum_{t=1}^{T} \sup_{z\in \calU(x_t)} \ind[\MAJ_{P_{t-1}}(z)\neq y_t] \leq a_\eta \min_{h \in \calH} \sum_{t=1}^{T} \sup_{z\in \calU(x_t)} \ind[h(z)\neq y_t] + b_\eta \ln\abs{\calH},\]
where $a_\eta= \frac{\ln(1/\eta)}{\log(2/(1+\eta))}$ and $b_\eta = \frac{1}{\ln(2/(1+\eta))}$. In particular, setting $1-\eta=\min\{(2\ln \abs{\calH})/T, 1/2\}$ yields 
\[\sum_{t=1}^{T} \sup_{z\in \calU(x_t)} \ind[\MAJ_{P_{t-1}}(z)\neq y_t] \leq 2\mathsf{OPT} + 4\sqrt{\mathsf{OPT}\ln\abs{\calH}}. \]
Furthermore, \alg{Weighted Majority} (\prettyref{alg:weighted-majority}) makes at most $T$ oracle queries to $\sfO_\calU$.
\end{lem}

\begin{proof}
This proof follows from standard analysis for the \alg{Weighted Majority} algorithm (see e.g.~\cite{robschapire,blum2007learning}). Let $W_t = \sum_{h \in \calH} P_t(h)$. Observe that on round $t$, if the weighted-majority-vote $\MAJ_{P_{t-1}}$ is not robustly correct on $(x_t, y_t)$, then:
\begin{align*}
    W_{t} &= \eta \sum_{h\in \calH: h(z_t)\neq y} P_{t-1}(h) + \sum_{h\in \calH: h(z_t) = y} P_{t-1}(h) = \eta \sum_{h\in \calH: h(z_t)\neq y} P_{t-1}(h) + W_{t-1} - \sum_{h\in \calH: h(z_t)\neq y} P_{t-1}(h)\\
    &= W_{t-1} - (1-\eta) \inparen{\sum_{h\in \calH: h(z_t)\neq y} P_{t-1}(h)} \leq W_{t-1} - (1-\eta)\frac{1}{2}W_{t-1} = \inparen{\frac{\eta+1}{2}} W_{t-1},
\end{align*}
where the last inequality follows from the fact that $\sum_{h\in \calH: h(z_t)\neq y} P_{t-1}(h) \geq  \sum_{h\in \calH: h(z_t) = y} P_{t-1}(h)$. 

Denote by $M = \sum_{t=1}^{T} \sup_{z\in \calU(x_t)} \ind[\MAJ_{P_{t-1}}(z)\neq y_t]$ the number of rounds on which the weighted-majority-vote was not robustly correct during the total $T$ rounds. The above implies that $W_T \leq \inparen{\frac{\eta+1}{2}}^MW_0 = \inparen{\frac{\eta+1}{2}}^M\abs{\calH}$. On the other hand, denote by $\mathsf{OPT}=\min_{h\in \calH} \sum_{t=1}^{T} \sup_{z\in \calU(x)} \ind[h(z)\neq y]$ the number of rounds on which the best predictor $h^*$ in $\calH$ was not robustly correct. Whenever the weighted-majority-vote is not robustly correct, $h^*$ might make a mistake on $(z_t,y_t)$. It follows that after $T$ rounds, $P_T(h^*) \geq \eta^{\mathsf{OPT}}$. Combining the above inequalities, we get
\[ \eta^{\mathsf{OPT}} \leq P_{T}(h^*) \leq W_T \leq \inparen{\frac{\eta+1}{2}}^M\abs{\calH},\]
and solving for $M$ yields 
\[M \leq \frac{\ln(1/\eta)}{\ln(2/(1+\eta))}\mathsf{OPT} +  \frac{1}{\ln(2/(1+\eta))}\ln\abs{\calH}.\]
To conclude the proof, observe that for $\eta\in [0,1), \ln(2/(1+\eta))\geq \frac{1-\eta}{2}$, and $\ln(1/\eta)\leq (1-\eta)+(1-\eta)^2$ for $0\leq 1-\eta \leq 1/2$. Setting $1-\eta=\min\{(2\ln \abs{\calH})/T, 1/2\}$ yields the desired bound.
\end{proof}

\begin{lem}
\label{lem:agnostic-littlestone}
For any class $\calH$ with finite Littlestone dimension $\Ldim(\calH) < \infty$ and integer $T$, let $\alg{Experts}_\calH = \sett{\alg{Expert}(i_1,\dots,i_L): 1 \leq i_1 <\cdots <i_L\leq T, L\leq \Ldim(\calH)}$ be a set of experts as described in \prettyref{alg:weighted-majority}. Then, running \alg{Weighted Majority} (\prettyref{alg:weighted-majority}) with $\alg{Experts}_\calH$ guarantees that for any perturbation set $\calU$ and any sequence of examples $(x_1,y_1),\dots, (x_T,y_T)$, 
\[\sum_{t=1}^{T} \sup_{z\in \calU(x_t)} \ind[\MAJ_{P_{t-1}}(z)\neq y_t] \leq 2 \mathsf{OPT} + 4\sqrt{\mathsf{OPT} \ln\abs{\alg{Experts}_\calH}},\]
where $\mathsf{OPT} = \min_{h \in \calH} \sum_{t=1}^{T} \sup_{z\in \calU(x_t)} \ind[h(z)\neq y_t]$, and $1-\eta=\min\{(2\ln \abs{\alg{Experts}_\calH})/T, 1/2\}$. Furthermore, \alg{Weighted Majority} (\prettyref{alg:weighted-majority}) makes at most $T$ oracle queries to $\sfO_\calU$.
\end{lem}

\begin{proof}
Let $\calU$ be an arbitrary adversary, and $(x_1,y_1), \dots, (x_T, y_T)\in \calX\times \calY$ be an arbitrary sequence. Let $h^*\in \calH$ be an optimal robust predictor on this sequence,
i.e. $\sum_{t=1}^{T} \sup_{z\in \calU(x_t)} \ind[h^*(z)\neq y_t] = \mathsf{OPT}$. Let $\texttt{Experts}_\calH = \sett{\texttt{Expert}(i_1,\dots,i_L): 1 \leq i_1 <\cdots <i_L\leq T, L\leq \Ldim(\calH)}$ denote the set of experts instantiated that simulate the Standard Optimal Algorithm as described in \prettyref{alg:weighted-majority}. 

Consider running \alg{Weighted Majority} (\prettyref{alg:weighted-majority}) with $\alg{Experts}_\calH$ as its finite cardinality set of experts on the sequence $(x_1,y_1),\dots, (x_T,y_T)$. Consider the set of perturbations returned by the perfect attack oracle $\sfO_\calU$ during the rounds on which the weighted-majority-vote was not robustly correct, 
\[Q = \sett{(z_t,y_t): 1 \leq t \leq T \wedge \MAJ_{P_{t-1}}\text{ is not robustly correct on }(x_t,y_t) }.\]

By \prettyref{alg:weighted-majority}, there is a choice of rounds $i^*_1 < \cdots < i^*_L$ such that $\texttt{Expert}(i^*_1,\dots,i^*_L) \in \texttt{Experts}_\calH$ agrees with the predictions of $h^*$ on this particular sequence $Q$. Observe that the number of mistakes $h^*$ makes on this sequence $M(h^*):=\abs{\sett{(z,y)\in Q: h^*(z)\neq y }}\leq \mathsf{OPT}$. Thus, the weight of $\texttt{Expert}(i^*_1,\dots,i^*_L)$ is at least $\eta^{M(h^*)} \geq \eta^{\mathsf{OPT}}$ (since $\eta < 1$). The remainder of the proof follows exactly as in the  proof of \prettyref{thm:agnostic-littlestone}. 
\end{proof}

\begin{proof}[of \prettyref{thm:agnostic-littlestone}]
Let $\calH\subseteq \calY^\calX$ be a hypotesis class with finite Littlestone dimension $\Ldim(\calH)<\infty$. Let $\calB:(\calX \times \calY)^*\to \calY^\calX$ denote the \alg{Weighted Majority} algorithm running with experts $\alg{Experts}_\calH$ as described in \prettyref{thm:agnostic-littlestone}. We will apply a standard online-to-batch conversion \cite{DBLP:journals/tit/Cesa-BianchiCG04} to get the desired result. Specifically, on input dataset $S=\sett{(x_j,y_j)}_{j=1}^{m}$ that is drawn iid~from some unknown distribution $\calD$ over $\calX\times \calY$, output a uniform distribution over hypotheses $\hat{h}_0,\hat{h}_1,\dots,\hat{h}_{m-1}$ where $\hat{h}_{i} = \calB(\sett{(x_j,y_j)}_{j=1}^{i-1})$. We are guaranteed that with probability at least $1-\delta$ over $S\sim \calD^m$, 
\begin{align*}
    \Ex_{j \sim \mathsf{Unif}\sett{0,\dots,m-1}} \insquare{ \Risk_{\calU}(\hat{h}_j;\calD) } &\leq \frac{1}{m}\sum_{j=1}^{m} \sup_{z\in\calU(x_j)} \ind[\hat{h}_{j-1}(z)\neq y_j] + \sqrt{\frac{2\ln (1/\delta)}{m}}\\
    &\leq 2\min_{h \in \calH} \frac{1}{m}\sum_{j=1}^{m} \sup_{z\in\calU(x_j)} \ind[h(z)\neq y_j] + 4\sqrt{\frac{\ln\abs{\alg{Experts}_\calH}}{m}} + \sqrt{\frac{2\ln (1/\delta)}{m}}\\
    &\leq 2\min_{h\in \calH} \Risk_{\calU}(h;\calD) + 4\sqrt{\frac{\ln\abs{\alg{Experts}_\calH}}{m}} + 2\sqrt{\frac{2\ln (1/\delta)}{m}}.
\end{align*}
This yields a sample complexity bound of $m(\eps,\delta)=O\inparen{\frac{\ln\abs{\alg{Experts}_\calH}+\ln(1/\delta)}{\eps^2}}$. The oracle complexity $T(\eps,\delta)=O(m(\eps,\delta)^2)$ since we invoke learner $\calB$ $m$ times on datasets of size at most $m$.
\end{proof}

\begin{proof}[of \prettyref{thm:agnsotic-vc-littlestone}]
This proof follows an argument originally made by \cite{david:16} to reduce agnostic sample compression to realizable sample compression in the non-robust setting, and later adapted by \cite{pmlr-v99-montasser19a} for the robust setting. Let $\calD$ be an arbitrary distribution over $\calX\times \calY$. Fix $\eps, \delta \in (0,1)$ and a sample size $m$ that will be determined later. Let $S=\sett{(x_1,y_1),\dots,(x_m,y_m)}$ be an iid sample from $\calD$. Denote by $\Tilde{\calB}$ the robust learning algorithm in the realizable setting from \prettyref{thm:realizable-strongust}, and denote by $\calA_{\rm cyc}$ \alg{\cycalg} (\prettyref{alg:robust_learner_unknown}) from \prettyref{thm:realizable-littlestone-2}. The proof is broken into two parts. 

\paragraph{Finding Maximal Subsequence $S'$ with Zero Robust Loss} We will use $\calA_{\rm cyc}$ to find a maximal subsequence $S'\subseteq S$ on which the robust loss can be zero, i.e.~$\inf_{h\in\calH}\Risk_{\calU}(h;S')=0$. This can be done by running $\calA_{\rm cyc}$ on all $2^m$ possible subsequences, with a total oracle complexity of $2^m\Ldim(\calH)$. 

\paragraph{Agnostic Sample Compression} We now run the boosting algorithm $\Tilde{B}$ on $S'$. \prettyref{thm:realizable-strongust} guarantees that the robust risk of $\hat{h}=\Tilde{B}(S',\sfO_\calU,\calA_{\rm cyc})$ is zero, $\Risk_{\calU}(\hat{h};S')=0$. Since $S'$ is a maximal subsequence on which the robust loss can be zero, this implies that
\[\Risk_{\calU}(\hat{h};S)\leq \min_{h\in \calH} \Risk_\calU(h;S).\]

Furthermore, the predictor $\hat{h}$ can be specified using $m_0O(\log\abs{S'})\leq m_0O(\log m)$ points from $S$, which is due to the robust compression guarantee in the proof of \prettyref{thm:realizable-strongust}. Now, we can apply agnostic sample compression generalization guarantees for the robust loss.

Similarly to the realizable case (see \prettyref{lem:robust-compression}), 
uniform convergence guarantees for sample compression 
schemes \cite{graepel:05} remain valid for the robust loss, 
by essentially the same argument; the essential argument is the 
same as in the proof of \prettyref{lem:robust-compression} except 
using Hoeffding's inequality to get concentration of the empirical 
robust risks for each fixed index sequence, and then a union bound over 
the possible index sequnces as before.
We omit the details for brevity. 
In particular, denoting $T_{m} = O(\log m)$, 
for $m > m_0 T_{m}$, with probability at least $1-\delta/2$, 
\begin{equation*}
\Risk_{\calU}(\hat{h};\calD)
\leq \hat{\Risk}_{\calU}(\hat{h};S) 
+ \sqrt{\frac{m_0T_{m} \ln(m) + \ln(2/\delta)}{2m - 2m_0 T_{m}}}.
\end{equation*}

Let $h^{*} = \argmin_{h \in \calH} \Risk_{\calU}(h;\calD)$ 
(supposing the min is realized, for simplicity; else 
we could take an $h^{*}$ with very-nearly minimal risk).
By Hoeffding's inequality, with probability at least 
$1-\delta/2$, 
\begin{equation*}
\hat{\Risk}_{\calU}(h^{*};S) 
\leq \Risk_{\calU}(h^{*};\calD) 
+ \sqrt{\frac{\ln(2/\delta)}{2m}}.
\end{equation*}

By the union bound, if $m \geq 2 \Mre T_{m}$, with probability at least $1-\delta$, 
\begin{align*}
\Risk_{\calU}(\hat{h};\calD) 
& \leq \min_{h \in \calH} \hat{\Risk}_{\calU}(h;S) 
+ \sqrt{\frac{m_0 T_{m} \ln(m) + \ln(2/\delta)}{m}}
\\ & \leq \hat{\Risk}_{\calU}(h^{*};S) 
+ \sqrt{\frac{\Mre T_{m} \ln(m) + \ln(2/\delta)}{m}}
\\ & \leq \Risk_{\calU}(h^{*};\calD) 
+ 2\sqrt{\frac{\Mre T_{m} \ln(m) + \ln(2/\delta)}{m}}.
\end{align*}
Since 
$T_{m} = O( \log(m) )$, 
the above is at most $\eps$ 
for an appropriate choice of sample size 
$m = O\!\left( \frac{\Mre}{\eps^{2}} \log^{2}\!\left(\frac{\Mre}{\eps}\right) + \frac{1}{\eps^{2}}\log\!\left(\frac{1}{\delta}\right) \right)$.
\end{proof}

\section{Lower Bound Proof for \prettyref{sec:lowerbound-mistakeoracle}}
\label{app:appendix-lowerbound}

\begin{proof}[of \prettyref{thm:lowerbound-oraclecomplexity}]
Let $d= \Tdim(\calH)$. By definition of the threshold dimension, $\exists P=\sett{x_1,\dots, x_d} \subseteq \calX$ that is threshold-shattered using $C=\sett{h_1,\dots, h_d}\subseteq \calH$. Let $\calD$ be a uniform distribution over $(x_1,+1)$ and $(x_d, -1)$. Let $\calB$ be an arbitrary learner in the \smodel. For any $h\in C\setminus \sett{h_d}$, let $\calU_h: \calX \to 2^\calX$  be defined as:
\begin{align*}
    \calU_h(x_1)&=\sett{ x \in P: h(x)=+1 },\\
    \calU_h(x_d)&= \sett{x \in P: h(x)= -1}= P \setminus \calU_h(x_1),\\
    \calU_h(x)&=\sett{x_0}~\forall x \in \calX \setminus \sett{x_1,x_d},
\end{align*}
where $x_0 \in \calX\setminus{P}$.

For any such $\calU_h$, observe that finding a predictor $\hat{h}: \calX \to \calY$ that achieves zero robust loss on $\calD$, $\Risk_{\calU_h}(\hat{h}; \calD)=0$, is equivalent to figuring out which threshold $h \in C\setminus {h_d}$ was used to construct $\calU_h$, since $\Risk_{\calU_h}(h; \calD)=0$ by definition of $\calU_h$, but for any other threshold $h' \in C\setminus {h_d}$ where $h' \neq h$, $\Risk_\calU(h'; \calD)\geq 1/2$. 

We will pick $h$ uniformly at random from $C\setminus {h_d}$, and we will show that in expectation over the random draw of $h$, learner $\calA$ needs to make at least $\Omega(\log \abs{C\setminus {h_d}})$ oracle queries to $\mathsf{O}_{\calU_{h}}$ in order to achieve robust loss zero on $\calD$. For ease of presentation, for each $i\in [d-1]$, we will encode $h_i\in C\setminus {h_d}$ with the binary representation $r(i)$ of integer $i$, for example:
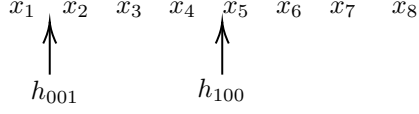
\begin{figure}[h]
    \centering
    \tikzset{every picture/.style={line width=0.75pt}} 

\begin{tikzpicture}[x=0.75pt,y=0.75pt,yscale=-1,xscale=1]

\draw    (331,177.66) -- (331,153) ;
\draw [shift={(331,151)}, rotate = 450] [color={rgb, 255:red, 0; green, 0; blue, 0 }  ][line width=0.75]    (10.93,-3.29) .. controls (6.95,-1.4) and (3.31,-0.3) .. (0,0) .. controls (3.31,0.3) and (6.95,1.4) .. (10.93,3.29)   ;
\draw    (244,177.66) -- (244,153) ;
\draw [shift={(244,151)}, rotate = 450] [color={rgb, 255:red, 0; green, 0; blue, 0 }  ][line width=0.75]    (10.93,-3.29) .. controls (6.95,-1.4) and (3.31,-0.3) .. (0,0) .. controls (3.31,0.3) and (6.95,1.4) .. (10.93,3.29)   ;

\draw (222,139.4) node [anchor=north west][inner sep=0.75pt]    {$x_{1} \ \ \ x_{2} \ \ \ x_{3} \ \ \ x_{4} \ \ \ x_{5} \ \ \ x_{6} \ \ \ x_{7} \ \ \ \ x_{8}$};
\draw (233,179.4) node [anchor=north west][inner sep=0.75pt]    {$h_{001}$};
\draw (317,178.4) node [anchor=north west][inner sep=0.75pt]    {$h_{100}$};

\end{tikzpicture}

\end{figure}

Thus, drawing $h$ uniformly at random from $C\setminus {h_d}$ is equivalent to drawing a random bit-string $r$ of length $\ceil{\log_2 \abs{C\setminus {h_d}}}$ bits. Next, we will define the behavior of the oracle $\sfO_{\calU_h}$. 

\begin{algorithm2e}[H]
\label{alg:mistake-oracle}
\caption{Perfect Attack Oracle $\sfO_{\calU_h}$}
\SetKwInput{KwInput}{Input}                
\SetKwInput{KwOutput}{Output}              
\SetKwFunction{FMain}{ZeroRobustLoss}
\DontPrintSemicolon
  \KwInput{A predictor $f:\calX \to \calY$ and a labeled example $(x,y)$.}
\KwOutput{Assert that $f$ is robustly correct on $(x,y)$ or return a $z\in\calU_h(x)$ such that $f(z)\neq y$.}
\If{$x=x_1$}{
    Output the first $z\in \calU_{h}(x_1)$ \textbf{(to the right of $x_1$)} such that $f(z)\neq y$. If no such $z$ exists, assert that $f$ is robustly correct on $(x_1,y)$.
}\ElseIf{$x=x_d$}{
    Output the first $z\in \calU_{h}(x_d)$ \textbf{(to the left of $x_d$)} such that $f(z)\neq y$. If no such $z$ exists, assert that $f$ is robustly correct on $(x_d,y)$.
}\Else{
    Output $x$ if $f(x)\neq y$, otherwise assert that $f$ is robustly correct on $(x,y)$.
}
\end{algorithm2e}

Before learning starts, from the perspective of the learner $\calB$, the version space $V_0= \sett{h_1,\dots, h_{d-1}}$, as any of these thresholds could be the true threshold used by $\sfO_{\calU_{h_r}}$ where $r$ was drawn uniformly at random. On each round $t$, learner $\calB$ constructs a predictor $h_t: \calX \to \calY$ and asks the oracle $\calO_{\calU_{h_r}}$ with a query $q_t = (h_t, (x_t,y_t))$, and the oracle $\sfO_{\calU_{h_r}}$ responds as described in Algorithm~\ref{alg:mistake-oracle}. Without loss of generality, we can assume that $(x_t,y_t)=(x_1, +1)$ or $(x_t,y_t)=(x_d,-1)$, since queries concerning other points $x\in \calX\setminus \sett{x_1,x_d}$ do not reveal helpful information for robustly learning distribution $\calD$. Foreach round $t$, the version space $V_t$ describes the set of thresholds that are consistent with the queries constructed by the learner so far, i.e. $\forall i\leq t,\forall h,h'\in V_t, \sfO_{\calU_{h}}(q_i) = \sfO_{\calU_{h'}}(q_i)$. So, from the perspective of the learner $\calB$, any $h\in V_t$ could be the true threshold. 

We will show that with each oracle query $q_t$ constructed by the learner $\calB$, in expectation over the random draw of $r$, the size of the newly updated version space $\abs{V_{t}} \geq \frac{1}{4} \abs{V_{t-1}}$. Formally, the expected size of the version space $V_t$ after round $t$ conditioned on query $q_t$ and $V_{t-1}$ is:
\[\Ex_{r_t} \insquare{\abs{V_t} | q_t, V_{t-1}} = \Prob_{r_t}\insquare{r_t = 0} \Ex{\insquare{\abs{V_t} | q_t, V_{t-1}, r_t}} + \Prob_{r_t}\insquare{r_t = 1} \Ex{\insquare{\abs{V_t} | q_t, V_{t-1}, r_t}},\]
where $r_t$ is the $t^{\text{th}}$ random bit in the random bit string $r$. We need to consider two possible cases depending on the query $q_t=(h_t, (x_t,y_t))$. (Without loss of generality, we are assuming that $h_t\in C\setminus{h_d}$, as the oracle $\sfO_{\calU_{h_r}}$ will treat it as such by Steps 2 and 4). 

If $(x_t,y_t)=(x_1, +1)$, and the $t^{\text{th}}$ bit of $h_t$ is $0$, then:
    \[\Ex_{r_t} \insquare{\abs{V_t} | q_t, V_{t-1}} \geq \Prob_{r_t}\insquare{r_t = 1} \Ex{\insquare{\abs{V_t} | q_t, V_{t-1}, r_t}} = \frac{1}{2}\cdot \frac{1}{2} \abs{V_{t-1}}. \]

If $(x_t,y_t)=(x_d, -1)$, and the $t^{\text{th}}$ bit of $h_t$ is $1$, then:
    \[\Ex_{r_t} \insquare{\abs{V_t} | q_t, V_{t-1}} \geq \Prob_{r_t}\insquare{r_t = 0} \Ex{\insquare{\abs{V_t} | q_t, V_{t-1}, r_t}} = \frac{1}{2}\cdot \frac{1}{2} \abs{V_{t-1}}. \]

Therefore, it follows that $\Ex_{r_t} \insquare{\abs{V_t} | q_t, V_{t-1}} \geq \frac{1}{4} |V_{t-1}|$. Thus, after $T$ rounds, $\Ex_{r}\insquare{\abs{V_{T}}| V_0}\geq \frac{1}{4}^T\abs{V_0}$. This implies that there exists a fixed bit-string $r^*$ (or equivalently, an adversary $\calU_{h_{r^*}}$) such that for $T \leq \frac{\log\abs{V_0}}{2}$ rounds, $\abs{V_T}\geq 1$. This implies that learner $\calB$ needs at least $\frac{\log\abs{V_0}}{2}$ oracle queries to $\mathsf{O}_{\calU_{h_{r^*}}}$ in order to robustly learn distribution $\calD$.
\end{proof}

\section{Proofs for \prettyref{sec:online}}
\label{app:appendix-c}

\begin{proof} [of \prettyref{thm:online-lowerbound}]
Observe that the same lowerbound construction from \prettyref{thm:lowerbound-oraclecomplexity} can be used in the setting of the online model. Specifically, by importing that construction, we get the following: there is a distribution $\calD$ over $\calX\times \calY$, such that for any learner $\calB:(\calX\times\calY)^*\to \calY^\calX$, there is a perturbation set $\calU:\calX\to2^{\calX}$ and an adversary $\bbA: \calY^\calX \times (\calX\times \calY) \to \calX$ (\prettyref{alg:mistake-oracle}) such that:
\[ M_{\calU,\bbA}(\calB,\calH;\calD)=\sum_{t=1}^{\infty} \ind \insquare{\calB(\sett{(z_i,y_i)}_{i=1}^{t-1})(z_t)\neq y_t} \geq \frac{\log_2{\inparen{\Tdim(\calH)-1}}}{2}.\]

This is because in the setting of the \smodel, learner $\calB$ \emph{chooses} which example $(x,y)\in {\rm supp}(\calD)$ to feed into $\bbA$, and still learner $\calB$ makes $\frac{\log_2{\inparen{\Tdim(\calH)-1}}}{2}$ mistakes before fully robustly learning distribution $\calD$. While in this setting, the examples $(x,y)$ that are fed into $\bbA$ are drawn iid from $\calD$, and so its at least as hard as the other setting. 
\end{proof}

\section{Proofs for \prettyref{sec:imperfect}}
\label{app:appendix-d}

\begin{algorithm2e}[H]
\caption{Robust Learner with Imperfect Attack.}\label{alg:robust-learner-attack}
\SetKwInput{KwInput}{Input}                
\SetKwInput{KwOutput}{Output}              
\SetKwFunction{FMain}{ZeroRobustLoss}
\DontPrintSemicolon
  \KwInput{$S=\sett{(x_1,y_1),\dots, (x_m,y_m)}$, $\eps ,\delta$, black-box conservative online learner $\calA$, black-box attacker $\bbA$.}
  Initialize $\hat{h} = \calA(\emptyset)$.\;
  \For{$1\leq i \leq m$}{
    Let $z_i=\bbA(\hat{h}, (x_i,y_i))$ be the perturbation returned by the attacker $\bbA$.\;
    If $\hat{h}$ is not correct on $(z_i, y_i)$, update $\hat{h}$ by running online learner on $\calA$ on $(z_i,y_i)$.\;
    Break when $\hat{h}$ is correct on a consecutive sequence of perturbations of length $\frac{1}{\eps}\log\inparen{\frac{\Ldim(\calH)}{\delta}}$.\;
  }
\KwOutput{$\hat{h}$.}
\end{algorithm2e}

\begin{proof}[of Theorem~\ref{thm:attack-generalization}]
Let $\calH \subseteq \calY^\calX$ be an arbitrary hypothesis class, and $\calA$ a conservative online learner for $\calH$ with mistake bound of $\Ldim(\calH)$. Let $\calU$ be an arbitrary adversary and $\bbA$ an arbitrary fixed (but possibly randomized) attack algorithm. Let $\calD$ be an arbitrary distribution over $\calX\times \calY$ that is robustly realizable,i.e.~$\inf_{h\in\calH}\Risk_\calU(h;\calD)=0$.

Fix $\eps, \delta \in (0,1)$ and a sample size $m = 2\frac{\Ldim(\calH)}{\eps} \log \inparen{\frac{\Ldim(\calH)}{\delta}}$. Since online learner $\calA$ has a mistake bound of $\Ldim(\calH)$, \prettyref{alg:robust-learner-attack} will terminate in at most $\frac{\Ldim(\calH)}{\eps} \log\inparen{\frac{\Ldim(\calH)}{\delta}}$ steps, which is an upperbound on the number of calls to the attack algorithm $\bbA$. It remains to show that the output of \prettyref{alg:robust-learner-attack}, the final predictor $\hat{h}$, will have low error w.r.t. future attacks from $\bbA$: $${\rm err}_{\bbA}(\hat{h};\calD)\triangleq \Prob_{\underset{\text{randomness of }\bbA}{(x,y)\sim\calD}}  \insquare{\hat{h}(\bbA(\hat{h}, (x,y))) \neq y}.$$ 

Throughout the runtime of \prettyref{alg:robust-learner-attack}, the online learner $\calA$ generates a sequence of at most $\Ldim(\calH)+1$ predictors. There's the initial predictor from Step 1, plus the $\Ldim(\calH)$ updated
predictors corresponding to potential updates by online learner $\calA$. By a union bound over these predictors, the probability that the final predictor $\hat{h}$ has error more than $\eps$ 
\[
    \Prob_{S\sim \calD^m} \!\!\insquare{ {\rm err}_{\bbA}(\hat{h};\calD)\!>\! \eps} \leq \!\!\Prob_{S\sim \calD^m} \!\insquare{ \exists j \!\in\! [\Ldim(\calH)\!+\!1]: {\rm err}_{\bbA}(h_j;\calD) \!>\! \eps}
    \!\leq\! (\Ldim(\calH)\!+\!1) (1-\eps)^{\frac{1}{\eps} \log\inparen{\frac{\Ldim(\calH)\!+1}{\delta}}}
    \!\leq\! \delta.
\]
Therefore, with probability at least $1-\delta$ over $S\sim \calD^m$, \prettyref{alg:robust-learner-attack} outputs a predictor $\hat{h}$ with error ${\rm err}_{\bbA}(\hat{h};\calD)\leq \eps$.
\end{proof}

\end{document}